\documentclass{article}
\usepackage[preprint]{neurips_2026}
\usepackage[utf8]{inputenc}
\usepackage[T1]{fontenc}
\usepackage[colorlinks=true, citecolor=blue, linkcolor=red, urlcolor=blue]{hyperref}
\usepackage{url}
\usepackage{booktabs}
\usepackage{amsfonts}
\usepackage{nicefrac}
\usepackage{microtype}
\usepackage{tikz}
\usetikzlibrary{shadows}
\usetikzlibrary{backgrounds}
\usetikzlibrary{arrows.meta}
\usepackage{amsmath, amssymb, amsthm}
\usepackage[dvipsnames]{xcolor}
\newtheorem{definition}{Definition}
\newtheorem{theorem}{Theorem}
\newtheorem{lemma}{Lemma}

\newtheorem{corollary}{Corollary}
\newtheorem*{theorem*}{Theorem}
\newtheorem*{informalthm*}{Informal Theorem}
\theoremstyle{remark}
\newtheorem{remark}{Remark}

\usepackage{enumitem}
\setlist[itemize]{leftmargin=1.15em,labelsep=0.35em,itemsep=0pt,topsep=2pt,parsep=0pt,partopsep=0pt}
\setlist[enumerate]{leftmargin=1.55em,labelsep=0.35em,itemsep=0pt,topsep=2pt,parsep=0pt,partopsep=0pt}

\allowdisplaybreaks
\definecolor{mygreen}{HTML}{1d9954}

\usepackage{mathtools}
\mathtoolsset{showonlyrefs}
\providecommand{\lrp}[1]{\ensuremath{\left( #1 \right)}}

\providecommand{\lrs}[1]{\ensuremath{\left[ #1 \right]}}
\providecommand{\lrc}[1]{\ensuremath{\left\{ #1 \right\}}}

\providecommand{\E}[1]{\ensuremath{\mathbb{E} \lrs{#1}}}
\providecommand{\BCA}[1]{\mathrm{BCA}\lrp{#1}}

\providecommand{\ceil}[1]{\ensuremath{\big\lceil #1 \big\rceil}}

\providecommand{\cB}{\mathcal{B}}
\providecommand{\cC}{\mathcal{C}}

\usepackage{graphicx}
\usepackage{float}
\usepackage{subcaption}
\usepackage{xfrac}
\usepackage[inkscapepath=build/svg-inkscape/]{svg}

\title{Beyond Fixed Points: Superpolynomial Capacity of Asymmetric Hopfield Networks}

\author{
\textbf{Aakash Kumar}$^{1}$ \and
\textbf{Anatoly Khina}$^{2}$ \and
\textbf{Frederik Mallmann-Trenn}$^{3,4}$ \and
\textbf{Emanuele Natale}$^{1}$ \\[1ex]
$^{1}$COATI, CNRS, Inria, I3S, Université Côte d'Azur, France \\
$^{2}$School of Electrical and Computer Engineering, Tel Aviv University, Israel \\
$^{3}$Department of Informatics, King's College London, UK \\
$^{4}$Institute of Science and Technology Austria (ISTA), Klosterneuburg, Austria \\[1ex]
\texttt{aakash.kumar@inria.fr} \quad
\texttt{anatolyk@tau.ac.il} \\
\texttt{frederik.mallmann-trenn@kcl.ac.uk} \\
\texttt{emanuele.natale@inria.fr}
}

\hypersetup{
    pdftitle={Beyond Fixed Points: Superpolynomial Capacity of Asymmetric Hopfield Networks},
    pdfauthor={Aakash Kumar, Anatoly Khina, Frederik Mallmann-Trenn, Emanuele Natale}
}

\begin{document}
\maketitle

\begin{abstract}
    Classical Hopfield networks are limited to static patterns due to symmetric weights, whereas asymmetric networks can encode temporal sequences via limit-cycle attractors.
Achieving high-capacity storage of long sequences in classical synchronous asymmetric networks, however, has remained a challenge.
We present a simple and robust construction within the classical asymmetric Hopfield model with binary neurons and synchronous updates, that allows $n$ neurons to support $\exp\!\big(\Omega(n/(\log n)^2)\big)$ distinct limit-cycle attractors, each with period $\exp\!\big(\Omega(\sqrt n/\log n)\big)$ and robust to random noise with flip probability up to $\frac12-o(1)$, yielding superpolynomial capacity in both the number and length of stored sequences.
This is the first demonstration of such capacity for asymmetric Hopfield networks, which we obtain by combining results from combinatorics, number theory and the analysis of opinion dynamics.
Our findings show that synchronous asymmetric Hopfield networks possess a sequence-memory capacity which is larger and more robust than previously recognized, demonstrating that, in both biological and artificial neural systems, robust sequence representation can be achieved through coarse architectural motifs rather than complex nonlinearities.

\end{abstract}

\section{Introduction}
    \label{sec:intro}
    Associative memory, the ability to retrieve complete patterns from partial or noisy cues, is a fundamental concept in both neuroscience and artificial intelligence. The Hopfield network, a recurrent neural network model, stands as a canonical example of how collective neuronal dynamics can give rise to associative memory \citep{hopfield1982neural}.
In its classic formulation, the network is characterized by a symmetric weight matrix. This symmetry guarantees the existence of a Lyapunov energy function, ensuring that the network's state always converges to a stable fixed point.\footnote{If the update rule is synchronous, the network may exhibit limit cycles of period two; however, limit cycles of higher period cannot occur.} 
These fixed points act as stored memories, making the symmetric Hopfield network an effective model for static pattern recognition.

\noindent However, the constraint of symmetry, while ensuring stability, fundamentally limits the network's dynamical repertoire to point attractors (see Section \ref{sec:hopfield_intro}). Many cognitive processes, such as motor control, central pattern generation, and sequential reasoning, are inherently dynamic and involve recalling and generating temporal sequences rather than static patterns.
By allowing for asymmetric connections, asymmetric Hopfield networks break the limitations imposed by a global energy function.
This asymmetry opens the door to a richer set of dynamics, including the emergence of limit cycles, i.e., periodic sequences of states that can serve as attractors.
These limit-cycle attractors, formalized in Section~\ref{sec:hopfield_intro} (Definition \ref{def:limit_cycle}), provide a natural mechanism for storing and retrieving temporal sequences. 
For instance, in biological systems, high-capacity limit cycles can act as robust Central Pattern Generators (CPGs) for complex rhythmic motor control, while in artificial intelligence, they can serve as high-dimensional ``clocks'' for temporal reasoning and sequence timing.
While the potential for asymmetric (or directed) networks to store sequences is known, and various construction methods have been proposed \citep{zhang2013storing, muscinelli2017exponentially}, the question of their storage capacity remains a critical area of investigation. \cite{muscinelli2017exponentially} demonstrated that asymmetric Hopfield networks with $n$ neurons can support limit-cycle attractors of length $2^n$, thereby realizing dynamics that traverse the entire state space.
However, this result relies on a highly specialized construction in which the weights are carefully engineered on a case-by-case basis.
As a consequence, while theoretically striking, the model is unlikely to provide a realistic account of memory storage in biological or practical settings.
A further challenge for such constructions is robustness.
In practical settings, one would expect the dynamics to reliably return to the same limit-cycle attractor after perturbations induced by noise.
However, if the dynamics traverse the entire state space, robustness is fundamentally unattainable.
Recent advances have pushed the capacity of Hopfield-type networks by modifying the standard model with nonlinear interaction terms or continuous states \citep{krotov2021large, ramsauer2021hopfield, chaudhry2023long}. This raises a key question:

\begin{center}
    \begin{minipage}{0.94\linewidth}
        \itshape
        How can one construct, within the classical synchronous asymmetric Hopfield framework, a model that not only encodes temporal sequences but also supports the reliable storage of a large number of long sequences, while remaining robust to noise and simple enough to be compatible with what could be realized through biological developmental processes?
    \end{minipage}
\end{center}

\noindent This paper addresses this challenge directly.
We introduce a novel construction for asymmetric Hopfield networks based on a very simple network topology, presented in Section~\ref{sec:main}.
We show that this architecture supports a superpolynomial number of distinct limit-cycle attractors (with respect to the number of neurons), with periods that can themselves be superpolynomial in the network size. Moreover, the proposed construction is inherently robust to even extreme levels of noise, admitting non-trivial basins of attraction around each cycle. Finally, we argue in Section \ref{sec:sparse} that the model satisfies natural biological desiderata: beyond coarse architectural constraints, the result is provably preserved under highly sparsified network topologies and addition of adversarial connections.
Throughout, we work with the parallel update rule introduced by Little \citep{little1974persistent}; we do not analyze the asynchronous single-neuron dynamics more commonly associated with the original symmetric Hopfield model, and extending comparable guarantees to that regime remains open.
This should be viewed as a classical variant of the same binary-threshold architecture rather than as a modern modification of the model: Little's rule uses the same linear local fields and sign activation, differing only in the update schedule. Since our goal is to study deterministic sequence storage through periodic attractors, the synchronous regime is also the natural classical setting in which those limit cycles appear as first-class dynamical objects.
We now present an informal statement of our main results.

\begin{informalthm*}[Informal statement of Theorems~\ref{thm:main} and~\ref{thm:main_reg_sparse}]
    There exists an asymmetric Hopfield network with $n$ neurons that supports $\exp \big(\Omega\big(\frac{n}{(\log n)^2}\big)\big)$ distinct limit-cycle attractors, each with period $\exp \big(\Omega(\sqrt n/\log n)\big)$.
    Moreover, these limit cycles remain robust under independent flips with probability up to $\frac12-\Omega\big(\frac 1{\sqrt{\log n}}\big)$, and hence up to $\frac12-o(1)$, \emph{(Theorem \ref{thm:main})}. Furthermore, this robustness extends to highly sparsified network topologies, with only a modest reduction in capacity and period length, both of which remain superpolynomial in $n$ \emph{(Theorem \ref{thm:main_reg_sparse})}; the corresponding quantitative noise thresholds are made explicit in the proofs.
\end{informalthm*}

\noindent We further show in Section \ref{sec:sparse} (Theorem \ref{thm:adv}) that the construction is robust to addition of adversarial connections in the graph, which strengthens the case for biological plausibility. 
In Appendix \ref{sec:main_results}, we also record complementary global convergence statements: in the dense construction every initial configuration rapidly reaches some limit cycle deterministically, in the sparse construction the same holds with high probability, and under adversarial block perturbations the trajectory reaches the corresponding weak limit cycle. 
We also provide empirical support to our results through experiments in Section \ref{sec:Expt}. Our main contributions are summarized below.
\begin{itemize}
\item
    \textbf{Classical Synchronous Asymmetric Hopfield Model:} We demonstrate that it is possible to achieve very large sequence memory capacity without altering the classical asymmetric Hopfield network model.
    Our construction uses standard binary neurons and the synchronous update rule \citep{little1974persistent}, in contrast to recent works that rely on nonlinear interactions or continuous states. 
    Our block-based architectural motif suggests that enforcing coarse block-cyclic topologies in modern ML sequence models (e.g., structured state-space models or sparse RNNs) could inherently boost robust sequence capacity.
\item
    \textbf{Superpolynomial Capacity and Period Length:} We provide a constructive proof that our network architecture can store a superpolynomially large number of distinct limit-cycle attractors, where the period of these attractors is also superpolynomially long in the number of neurons.
\item
    \textbf{High Simplicity and Inherent Robustness:} Up to coarse architectural constraints, the result is preserved under highly sparsified connectivity and addition of adversarial connections,  enhancing biological plausibility relative to highly fine-tuned alternatives. Moreover, our construction is robust, ensuring stable retrieval even under extreme perturbations by random flips with probability up to $\frac12-\Omega(\sfrac 1{\sqrt{\log n}})$ in the dense case, and more generally up to $\frac12-o(1)$ once the block size satisfies $d=\omega(\log n)$.
\end{itemize}

\section{Related Work}
\vspace{-0.7em}
    \label{sec:related_work}
    The study of sequence memory in recurrent neural networks has a rich history. Within the context of Hopfield-type networks, the move from symmetric to asymmetric weight matrices was the crucial step that enabled the modeling of temporal dynamics. Early work focused on Hebbian-style learning rules to embed sequences, but often suffered from limited capacity.

Classic results on the capacity of asymmetric Hopfield models include \cite{bastollaAttractorsFullyAsymmetric1997}, who showed that the dynamics can exhibit $O(n)$ limit cycles with periods exponential in $n$. This picture was refined by \cite{hwangNumberLimitCycles2019}, who (informally) characterized the distribution of these long cycles. In a different regime, \cite{hwangNumberLimitCycles2020} demonstrated that an exponential number of limit cycles of constant length $O(1)$ can arise.

\noindent More targeted construction methods have been developed to store specific cycles. For instance, \cite{zhang2013storing} studied when a prescribed cycle is admissible, constructed corresponding connectivity matrices via the pseudoinverse learning rule, and analyzed the topology of the resulting networks. Closer to our work, \cite{muscinelli2017exponentially} demonstrated a method to construct a standard asymmetric Hopfield network that produces a stable orbit of maximal length ($2^n$), covering the entire state space.
While this highlights the potential for long sequences, our approach differs by focusing on the combinatorial generation of a large number of distinct, long limit-cycle attractors.
Unlike the single maximal-length orbit, which requires highly specific weight assignments, our construction relies on coarse architectural constraints rather than precise tuning, making it more appealing from the standpoint of biological plausibility.

\noindent Separately, a recent revolution in associative memory has come from ``Modern Hopfield networks" or ``Dense Associative Memories" \citep{krotov2021large, ramsauer2021hopfield}. These models modify the energy function of the classical Hopfield network model, often by introducing higher-order or non-polynomial interactions. This modification allows for a storage capacity for static patterns that is exponential in the number of neurons, a dramatic improvement over the linear capacity of the original model. These modern Hopfield networks have also been shown to be equivalent to the attention mechanism in Transformers \citep{ramsauer2021hopfield}.

\noindent Naturally, an effort has been made to extend the benefits of these high-capacity models to the domain of sequence memory. \cite{chaudhry2023long} adapted the Dense Associative Memory framework to store and retrieve long sequences, achieving a significantly improved sequence capacity by leveraging a nonlinear interaction term. A key distinction of our work is that we do not modify the fundamental neuron model or its update rule.
We remain within the classical framework of the synchronous asymmetric Hopfield network with binary neurons and majority vote update rule. Our contribution is to show that, through the realization of certain families of graph topologies, it is possible to achieve a large number of long periodic sequences without resorting to the nonlinearities or continuous states that characterize modern Hopfield networks. A detailed side-by-side comparison is deferred to Table~\ref{tab:related-work-comparison} in Appendix~\ref{app:comparison-experiments}.

\section{Setup: Asymmetric Hopfield Networks}
    \label{sec:hopfield_intro}
    We begin by introducing the basic framework of asymmetric Hopfield networks and the dynamical notions relevant to our analysis. For convenience, Table~\ref{tab:notation-summary} in Appendix~\ref{app:comparison-experiments} collects the recurring notation used throughout the paper.

An asymmetric Hopfield network consists of $n$ binary neurons with state vector $x(t) \in \{-1,+1\}^n$ evolving in discrete time. The network is specified by a (not necessarily symmetric) weight matrix $W \in \mathbb{R}^{n \times n}$, where $w_{ij}$ represents the influence of neuron $j$ on neuron $i$, and we assume $w_{ii} = 0$ for all $i$. The input to neuron $i$ at time $t$ is
$h_i(t) = \sum_{j=1}^n w_{ij} x_j(t)$. 
We consider synchronous majority-rule dynamics, where all neurons update simultaneously according to 
$x_i(t+1) = \operatorname*{sign}(h_i(t))$, with the convention that $\operatorname{sign}(0) = x_i(t)$, so that neurons maintain their state in the case of a tie. 
This defines a deterministic update map $F : \{-1,+1\}^n \to \{-1,+1\}^n$.

\begin{definition}[Orbit]
Let \(F:\{-1,+1\}^n\to\{-1,+1\}^n\) denote the synchronous update map of the network. The \emph{orbit} or \emph{trajectory} starting from an initial state \(x\in\{-1,+1\}^n\) is the sequence
\(x,\;F(x),\;F^2(x),\;F^3(x),\ldots\)
\end{definition}

\begin{definition}[Limit cycle and basin of attraction]
\label{def:limit_cycle}
A \emph{limit cycle} of period \(P\geq 1\) is a set of distinct states \(\cC = \{x^{(0)},x^{(1)},\ldots,x^{(P-1)}\}\) such that, for the synchronous update map $F: \lrc{-1, 1}^n \to \lrc{-1, 1}^n$, we have \(F(x^{(i)})=x^{(i+1)}\) for \(i=0,\ldots,P-2\) and \(F(x^{(P-1)})=x^{(0)}\). The special case \(P=1\) is a fixed point. The \emph{basin of attraction} of \(\cC\) is \(\cB(\cC)=\{y\in\{-1,+1\}^n: \exists T \geq 0 \text{ such that } F^T(y)\in \cC\}\). Clearly, $\cC \subseteq \cB(\cC)$;
we call \(\cC\) a \emph{limit-cycle attractor} if \( \cC \subsetneq \cB(\cC)\).
\end{definition}

Since the state space is finite, every trajectory eventually becomes periodic. The long-term behavior is therefore described by fixed points and limit cycles.

In contrast to the symmetric case, asymmetric Hopfield networks generally do not admit a global energy function, and their dynamics can exhibit nontrivial periodic behavior. The state space is thus partitioned into disjoint basins corresponding to different attractors. To capture robustness under partial recovery, we introduce a weaker notion of limit cycles that allows a bounded number of neuron mismatches. We start with the notion of Block-max Hamming distance.

\begin{definition}[Block-max Hamming distance]
\label{def:block_max_hamming}
Let $\Sigma$ be an alphabet, let $r,d \in \mathbb{N}$, and let $\mathcal S$ be a partition of $\{1,\dots,rd\}$ into $r$ parts of size $d$. After fixing an ordering of the parts in $\mathcal S$, any two strings $u,v \in \Sigma^{rd}$ decompose as
\(u = \lrp{u_1, u_2, \dots, u_r}\) and \(v = \lrp{v_1, v_2, \dots, v_r}\),
where $u_j, v_j \in \Sigma^d$ are the $j$th blocks, for $j = 1,\dots,r$. 
Let $d_H$ denote the Hamming distance on $\Sigma^d$. The block-max Hamming distance between $u$ and $v$ with respect to $\mathcal S$ is defined as
\(d_{\mathrm{BM}}^{\mathcal S}(u,v) = \max_{1 \le j \le r} d_H\big(u_j, v_j\big)\).
\end{definition}

In our block-cyclic architecture constructions, $\mathcal S$ is the canonical partition into blocks of size $d$. We therefore write $d_{\mathrm{BM}}^{(d)}$ for the corresponding block-max Hamming distance.

\begin{remark}
The block-max Hamming distance is an $\ell_\infty$-type block metric: it controls the worst number of mismatches inside any single block, rather than the total number of mismatches across the network. This viewpoint is reminiscent of locality notions in coding theory, such as locally recoverable codes, where the value of a coordinate can be recovered by querying a small prescribed set of other coordinates \citep{gopalanLocalityCodewordSymbols2012}. 
\end{remark}

\begin{definition}[$(k,d)$-weak limit cycle]
\label{def:weak}
Let $F:\{-1,+1\}^n\to\{-1,+1\}^n$ be a synchronous update map, and let $\mathcal C=(c_0,c_1,\dots,c_{P-1})$ be a cyclic reference sequence of distinct states.
The sequence need not itself be an exact limit cycle of $F$.
Assume that the neurons are partitioned into $r$ disjoint blocks $S_1,\dots,S_r$, each of size $d$, so that $n=rd$.
After fixing an ordering of the neurons inside each block, this partition induces a decomposition of any state $x\in\{-1,+1\}^n$ as
\(x = (x_1,\dots,x_r)\), where \(x_j\in\{-1,+1\}^d\).
Let $d_{\mathrm{BM}}^{(d)}$ be the corresponding block-max Hamming distance.

We say that $\mathcal C$ is a \emph{$(k,d)$-weak limit cycle} for $F$ if for every timestep of the reference sequence $i\in\{0,\dots,P-1\}$ there exists a neighborhood
\(\mathcal N_i \subseteq \{-1,+1\}^n\) with \(c_i\in \mathcal N_i\),
such that for every initial state $\tilde x(0)\in \mathcal N_i$, there exist integers $T\ge 0$ and $s\in\{0,\dots,P-1\}$ satisfying
\begin{equation*}
d_{\mathrm{BM}}^{(d)}\!\bigl(F^{T+t}(\tilde x(0)),\,c_{s+t \bmod P}\bigr)\le k
\quad \text{for all } t\ge 0.
\end{equation*}
When the reference sequence is an exact limit cycle of $F$, this reduces to the usual cycle notion with blockwise error tolerance $k$.
\end{definition}

A central object in this work is that of noise-robust limit cycles.

\begin{definition}[Noise robustness]
Let $\mathcal{C} = \{x^{(0)}, \dots, x^{(P-1)}\}$ be a limit cycle of an asymmetric Hopfield network with update map $F$. We say that $\mathcal{C}$ is $(p,\rho)$-noise robust if, for $\tilde{x}(0)$ obtained from any $x^{(i)} \in \mathcal{C}$ by independently flipping each neuron (from $1$ to $-1$ or from $-1$ to $1$) with probability~$p$,
\(\Pr\!\left[ \exists \, T \ge 0 \text{ such that } F^{T}(\tilde{x}(0)) \in \mathcal{C} \right] \ge 1 - \rho\),
where the probability is over the noise (and any randomness in the network, if applicable).
\end{definition}

\section{Analytical Tools}
    \label{sec:req_math}
    In this section, we present several preliminary results that will be used in the proof of our main theorem. In Subsection~\ref{subsec:necklace}, we recall results concerning the counting of aperiodic binary sequences, while in Subsection~\ref{subsec:lcm}, we review results related to the growth of the least common multiple ($\mathrm{lcm}$) of random integers; further details can be found in Appendices~\ref{sec:necklace} and ~\ref{sec:lcm}, respectively.

\subsection{Aperiodic Binary Sequence Counting}
\label{subsec:necklace}
We briefly recall standard asymptotic estimates related to the counting of aperiodic binary sequences, that will be used later in the argument. Binary sequences arise naturally in several combinatorial problems and the distinction between periodic and aperiodic necklaces plays a key role in enumeration.
\begin{definition}[Aperiodic binary sequence]
Let $\ell \ge 1$ and let $x \in \{-1,+1\}^\ell$. For \mbox{$k \in \{0,1,\dots,\ell-1\}$}, define the cyclic shift $\sigma^k(x)$ by
\(\lrp{\sigma^k(x)}_i = x_{(i+k)\bmod \ell}\) for \(i=0,1,\dots,\ell-1\).
The sequence $x$ is said to be \emph{aperiodic} if
\(\sigma^k(x) \neq x\) for all \(1 \le k < \ell\),
equivalently, if all its cyclic shifts $\{\sigma^k(x) : 0 \le k < \ell\}$ are pairwise distinct.
\end{definition}
We now state an asymptotic estimate for the number of aperiodic binary sequences.
\begin{lemma}
    \label{lem:necklace}
    The number $\mathcal{M}_2(\ell)$ of aperiodic binary sequences of length $\ell$ is given by \(\mathcal{M}_2(\ell) = {2^\ell}\left(1+o(1)\right)\) as $\ell\to\infty$.
\end{lemma}
For a proof of Lemma \ref{lem:necklace}, see Appendix \ref{sec:necklace}.
\begin{remark}
A randomly chosen binary sequence of length $\ell$ is aperiodic with probability tending to 1 as $\ell \to \infty$.
\end{remark}

\subsection{LCM of Random Sets of Positive Integers}
\label{subsec:lcm}
\noindent The function $\psi(m) = \log \mathrm{lcm}\{1 \leq a \leq m\}$ was introduced by Chebyshev in his study of the distribution of prime numbers. A well-known consequence of the Prime Number Theorem is that $\psi(m) \sim m$ as $m \to \infty$, where by $a(m) \sim b(m)$ we mean that $\lim_{m \to \infty} a(m) / b(m) = 1$.

\cite{Cilleruelo2014} extended this analysis to random subsets of $\{1,2,\dots,m\}$ as follows. For a subset $A \subseteq \{1,2,\dots,m\}$, define $\psi(A) = \log \mathrm{lcm}\{a : a \in A\}$. Then, the following result holds for random subsets of $\{1,2,\dots,m\}$.

\begin{theorem}
\label{thm:cill_main}
Let $A$ be the set constructed by choosing each element from $\{1,2,\dots,m\}$ independently with probability $\delta = \delta(m)$. If $\delta=\delta(m)<1$ and $\delta \cdot m \to \infty$, then
\begin{equation}
    \label{eq:psi:asymptotics}
    \psi(A) \sim m \frac{\delta \log(\delta^{-1})}{1-\delta}
\end{equation}
asymptotically almost surely as $m \to \infty$.
\end{theorem}

Using the analysis of \cite{Cilleruelo2014}, we show that the same asymptotic result holds when each element is chosen independently from $\{b,\dots,m\}$ with probability $\delta$, where $b<m$ satisfies $b/m = o(1)$.

\begin{theorem}
\label{thm:cill_ours}
Let $A$ be the set constructed by choosing each element from $\{b,\dots,m\}$ independently with a constant probability $\delta$, for some integer $b=b(m)$ satisfying $b/m = o(1)$. Then, \eqref{eq:psi:asymptotics} holds asymptotically almost surely as $m \to \infty$.
\end{theorem}
In particular, in both cases, $\log \mathrm{lcm}(A)$ grows linearly in $m$ asymptotically almost surely. For further details, see Appendix \ref{sec:lcm}. Note that the above results imply that in both cases, $\mathrm{lcm}(A) = \exp(\Omega(m))$.

\section{Super-Polynomial Capacity of Asymmetric Hopfield Networks}
    \label{sec:main}
    In this section, we introduce the construction underlying our theoretical results. We then present our main results along with proof sketches; full proofs are deferred to Appendix \ref{sec:main_results}. In Section \ref{sec:sparse}, we discuss several relaxations of this construction, highlighting its biological plausibility. We begin with some definitions.

\begin{definition}[Blocks and block cycles]
A \emph{block} of size $d$ is a collection of $d$ neurons with no internal connections between them.
Given two blocks $S_1$ and $S_2$, we write $S_1 \to S_2$ if every neuron in $S_1$ is connected to every neuron in $S_2$, with all edges directed from $S_1$ to $S_2$, and each such edge assigned weight $1$ (see Figure \ref{fig:supnode_conn} in Appendix~\ref{app:construction-figures}).
A \emph{block cycle} of length $\ell$ is a collection of blocks $S_0,S_1,\dots,S_{\ell-1}$ (depicted Figure \ref{fig:sup_cycle} in Appendix~\ref{app:construction-figures}) such that
\[
S_0 \to S_1 \to \dots \to S_{\ell-1} \to S_0 .
\]
\end{definition}

\begin{definition}[Dense block-cyclic architecture]
Let $m$ and $d$ be positive integers.
A \emph{block-cyclic architecture with parameters $d$ and $m$}, denoted $\BCA{d,m}$, is an asymmetric Hopfield network constructed as follows.

\begin{enumerate}
\item 
    Let $A$ be a multiset of $z=\lceil 4m\rceil$ integers chosen independently uniformly at random from the set $\{1,2,\dots,m\}$.
\item 
    For each $\ell \in A$, create a \emph{block cycle} of length $\ell$ of blocks of size $d$.
\item 
    Let $N_{\mathrm{blocks}} = \sum\limits_{\ell \in A} \ell$ be the total number of blocks, and denote by $n = dN_{\mathrm{blocks}}$ the total number of neurons in the \emph{block-cyclic architecture}.
\end{enumerate}
\end{definition}

The total number of blocks is $N_{\mathrm{blocks}}=\Theta(m^2)$ with high probability, and hence the total number of neurons in $\BCA{d,m}$ is $n = dN_{\mathrm{blocks}} = \Theta \lrp{m^2 d}$ with high probability, as shown in Lemma~\ref{lem:sum_distinct} in Appendix~\ref{sec:sum}.
Also note that, after removing repeated samples from $A$, the resulting set contains an independently sampled subset of $\{1,\dots,m\}$ with constant inclusion probability with high probability; see Lemma~\ref{lem:linear_sampling_domination} in Appendix~\ref{sec:sum}.
We now state our main result.
\begin{theorem}
\label{thm:main}
Let $m\in \mathbb{Z}^{+}$ be sufficiently large. Fix a sufficiently large constant $K$, set $d=\lceil K(\log m)^2\rceil$. Let $n$ denote the number of neurons in $\BCA{d,m}$. Then there exist constants $c_1,c_2>0$ such that, with probability $1-o(1)$, $\BCA{d,m}$ contains a family $\mathcal A$ of limit-cycle attractors satisfying the following:
\begin{enumerate}
    \item 
    \label{itm:num_lim_cycles}
        (\textbf{Number of limit cycles})
        Writing $N_{\mathrm{cycles}}:=|\mathcal A|$, we have
        \begin{equation}
            \label{eq:num_lim_cycles}
            N_{\mathrm{cycles}} \ge \exp\!\Big( c_1 \frac{n}{(\log n)^2}\Big).
        \end{equation}

    \item 
    \label{itm:cycle_length}
        (\textbf{Periods of the constructed attractors})  
        Every attractor $\mathcal C\in\mathcal A$ has period at least 
        \begin{equation}
            \label{eq:cycle_length}
            \exp\!\Big( c_2 \frac{\sqrt n}{\log n}\Big).
        \end{equation}

    \item 
    \label{itm:robustness}
        (\textbf{Noise robustness})  
        For $n$ large enough, fix any $\mathcal C\in\mathcal A$ and any state $x\in\mathcal C$. If each neuron of $x$ is independently flipped with probability $p_n$ at time $t=0$, where 
        $p_n \le \frac12 - \sqrt{\sfrac{\log n}{d}}$,
        then with probability $1-o(1)$ the perturbed trajectory returns to $\mathcal C$.
\end{enumerate}
\end{theorem}

\begin{remark}
The choice $d=\lceil K(\log m)^2\rceil$ is made for concreteness. More generally, the same proof shows that if the block size $d=d(n)$ satisfies $d=\omega(\log n)$, then the dense construction remains noise robust for flip probabilities
\(p_n \le \frac12 - \sqrt{\sfrac{\log n}{d}} = \frac12-o(1)\),
while the bounds in \eqref{eq:num_lim_cycles} and \eqref{eq:cycle_length} become $N_{\mathrm{cycles}} \ge \exp(\Omega(n/d))$ and $P \ge \exp(\Omega(\sqrt{n/d}))$.
\end{remark}

We next provide a proof sketch of Theorem~\ref{thm:main}. A rigorous proof thereof is deferred to Appendix~\ref{sec:main_results}.

\begin{proof}[Proof Sketch]
\textit{1.} 
    Consider a block cycle in the architecture. Assign signs to the neurons such that neurons in the same block have all the same sign and such that the resulting binary sequence of blocks is aperiodic. Such a configuration forms a limit cycle, where blocks simply transfer their state to the succeeding one, making a rotating loop. The desired lower bound~\eqref{eq:num_lim_cycles} then follows by bounding the number of such configurations using results of Section~\ref{subsec:necklace}. 

\noindent
\textit{2.}
For every attractor in the constructed family, the full block-cyclic architecture evolves as the product of rotations on its block cycles. Since the binary string on each block cycle is aperiodic, the period equals the $\mathrm{lcm}$ of the block-cycle lengths, which can be bounded as in \eqref{eq:cycle_length} by appealing to results in Section~\ref{subsec:lcm}. 

\noindent
\textit{3.}
The choice of parameters guarantees that if each neuron is independently flipped with probability at most $\frac12-\sqrt{\sfrac{\log n}{d}}$, then the majority of neurons within each block is preserved with probability $1-o(1)$ by Hoeffding's bound, meaning that random perturbations are corrected at the next synchronous update.
%
%
\end{proof}

\begin{remark}
Appendix \ref{sec:main_results} also records a global, fast convergence statement for arbitrary initial configurations. 
For the dense construction of Theorem \ref{thm:main}, every trajectory reaches some limit cycle after at most one full turn around the longest block cycle; sparse and adversarial analogues are given in Theorem \ref{thm:global_cycle}.
\end{remark}

\begin{remark}
Even if the block-cycle lengths are chosen more carefully so as to be pairwise co-prime, the period does not improve drastically. For justification, see Remark \ref{rem:coprime} in Appendix \ref{sub_app:coprime}.
\end{remark}

\section{Sparsity, Randomness and Biological Plausibility}
    \label{sec:sparse}
    In this section, we present several results that highlight the biological plausibility of our theory. The adjacent blocks in the block cycles of Theorem \ref{thm:main} are complete bipartite graphs. We first show that this assumption can be relaxed: instead of all-to-all connectivity between adjacent blocks, each neuron may receive only $h$ incoming edges from the preceding block, for any odd $h \ge 3$.

\begin{definition}[$h$-sparse connection]
\label{def:spars_sup_nodes_conn}
Let $S_1$ and $S_2$ be two blocks. We say that $S_1$ is \emph{$h$-sparsely connected} to $S_2$ if every neuron in $S_2$ is connected to exactly $h$ neurons in $S_1$ chosen uniformly at random with replacement, with all edges directed from $S_1$ to $S_2$, and each edge assigned weight equal to its multiplicity, i.e., if a neuron in $S_1$ is chosen $v$ times, then the corresponding edge has weight $v$.
\end{definition}

\begin{definition}[Sparse block-cyclic architecture]
Let $m,b,h,d\in \mathbb{N}$ be fixed parameters. Suppose $3\leq h<d$ is an odd integer and $b<m$. An \emph{$h$-sparse block-cyclic architecture with parameters $h$, $d$, $b$, and $m$}, referred to as $h$-sparse $\BCA{d,b,m}$, is an asymmetric Hopfield network constructed as follows.

\begin{enumerate}
\item
    Let $A$ be a multiset of \(z=\lceil 4(m-b+1)\rceil\) integers chosen independently uniformly at random from the set $\{b,b+1,\dots,m\}$.
\item 
    For each $\ell \in A$, create a \emph{block cycle} of length $\ell$ of blocks of size $d$ where adjacent blocks in these block cycles are $h$-sparsely connected.
\item 
    Let $N_{\mathrm{blocks}} = \sum\limits_{\ell \in A} \ell$ be the total number of blocks, and denote by $n = dN_{\mathrm{blocks}}$ the total number of neurons in the \emph{sparse block-cyclic architecture}.
\end{enumerate}
\end{definition}
We now state the $h$-sparse version of the result. To prove this result, we have used the classic theory on consensus reaching in opinion dynamics, which is reviewed in Appendix \ref{sec:majority}.
\begin{theorem}
\label{thm:main_reg_sparse}
Let $\varepsilon \in(0,1)$ satisfy $\varepsilon(1+\xi) > 1$, let $h\in \mathbb{N}$ be odd with $h\geq 3$, and let $m\in \mathbb{N}$ be sufficiently large.\footnote{$\xi$ denotes the exponent in the $h$-majority recovery bound under the bias induced by the noise level $p$: starting from that bias, the process returns to the correct consensus in $O(\log d)$ rounds with probability $1-1/d^{\xi}$. See Appendix \ref{sec:majority} for details.}
Set $d=\lceil m^{2\varepsilon/(1-\varepsilon)}\rceil$ and choose $b=b(m)$ such that $b\ge c_3\log m$ for a sufficiently large constant $c_3$ and $b/m\to0$.
Let $n$ denote the number of neurons in the resulting $h$-sparse $\BCA{d,b,m}$.
Then there exist constants $c_1,c_2>0$ such that, with probability $1-o(1)$, the resulting $h$-sparse block-cyclic architecture contains a family $\mathcal A$ of limit-cycle attractors satisfying the following:

\begin{enumerate}
\item
\label{itm:sparse:num_lim_cycles}
(\textbf{Superpolynomial number of limit cycles})  
    Writing $N_{\mathrm{cycles}}:=|\mathcal A|$, we have \(N_{\mathrm{cycles}} \ge \exp\!\big( c_1 {n}^{1-\varepsilon}\big)\).

\item 
\label{itm:sparse:cycle_length}
(\textbf{Periods of the constructed attractors})  
    Every attractor $\mathcal C\in\mathcal A$ has period at least \(\exp\!\big( c_2 (n^{1-\varepsilon})^{\frac{1}{2}} \big)\).

\item 
\label{itm:sparse:robustness}
(\textbf{Noise robustness})  
    For $n$ large enough, fix any $\mathcal C\in\mathcal A$ and any state $x\in\mathcal C$. If each neuron of $x$ is independently flipped at time $t=0$ with probability $p$, where $p \le \frac12 - C_p\sqrt{\sfrac{\log n}{d}}$ for a sufficiently large constant $C_p$, then with probability $1-o(1)$ (as $n\to\infty$) the perturbed trajectory returns to $\mathcal C$.
\end{enumerate}
\end{theorem}

\begin{proof}[Proof Sketch]
\noindent
\textit{1.}
    Item~\ref{itm:sparse:num_lim_cycles} follows from the counting argument for Item~\ref{itm:num_lim_cycles} of Theorem~\ref{thm:main}, with $d = \ceil{K \log^2 m}$ replaced by $d = \ceil{m^\frac{2\varepsilon}{1-\varepsilon}}$.
    The same quotient by the global orbit period appears, but its logarithm is \(O(m)=o(n/d)\), so it does not affect the stated lower bound.
    Let $\mathcal A$ denote the corresponding family of attractors obtained from the aperiodic assignments.

\noindent
\textit{2.}
    Claim~\ref{itm:sparse:cycle_length} is proved by applying Lemma~\ref{lem:linear_sampling_domination} on the interval $\{b,\dots,m\}$ and then using Theorem~\ref{thm:cill_ours}.

\noindent
\textit{3.}
    Because each neuron receives exactly $h$ inputs from the preceding block and $h$ is odd, the Hopfield update on a block coincides with synchronous $h$-majority dynamics. If each neuron is flipped with probability $p \le \frac12 - C_p\sqrt{\sfrac{\log n}{d}}$, then each neuron remains unflipped with probability at least $\frac12 + C_p\sqrt{\sfrac{\log n}{d}}$; a Hoeffding bound therefore shows that every block still has bias at least $\lambda_1\sqrt{d\log d}$ toward its pre-noise state with probability $1-o(1)$. The large-bias recovery estimate for $h$-majority reviewed in Appendix \ref{sec:majority} then implies that each block returns to its original state within $O(\log d)$ rounds with failure probability at most $d^{-\xi}$, where $\xi$ is the exponent furnished by that recovery bound for this noise level. A union bound over all blocks yields overall recovery with probability $1-o(1)$. See Appendix \ref{sec:main_results} for the formal proof.
\end{proof}

We now state our result with added adversarial connections. In this result, we use the weak notion of limit cycles (Definition \ref{def:weak}). Appendix \ref{sub_app:spectral_embed} includes spectral embeddings of representative sparse block-cyclic architectures with adversarial connections. Let $\mathcal A$ be the class of limit cycles on the sparse block-cyclic architecture obtained from the aperiodic monochromatic assignments used in Theorem \ref{thm:main_reg_sparse}. We show that if a sufficiently small multiple of \(\sqrt d/\log d\) neurons in each block receive adversarial incoming edges from elsewhere in the block-cyclic architecture, and each neuron is independently flipped with probability $p\in(0,1/2)$, then the trajectory tracks the original cycle within block-max distance \(O(\sqrt d/\log d)\) with high probability.

\begin{theorem}[Robustness under adversarial perturbations]
\label{thm:adv}
Let $\varepsilon\in(0,1)$ satisfy $\varepsilon(1+\xi)>1$, where $\xi>0$ is the polynomial-decay exponent in the per-block recovery bound from Appendix \ref{sec:majority}, and let $m\in\mathbb N$ be sufficiently large.
Set $d=\big\lceil m^{2\varepsilon/(1-\varepsilon)}\big\rceil$ and choose $b=b(m)$ such that $b\ge c_4(\log m)^2$ for a sufficiently large constant $c_4$ and $b/m\to0$.
Let $\mathcal H$ be a $3$-sparse $\BCA{d,b,m}$ with $n$ neurons, and let $\mathcal A$ denote the set of its limit-cycle states obtained from aperiodic monochromatic assignments.
Fix any initial state $c\in\mathcal A$.
Choose a sufficiently small constant $\gamma>0$ and a sufficiently large constant $C_{\mathrm{adv}}>0$, and set
\[
q=\big\lfloor \gamma \frac{\sqrt d}{\log d}\big\rfloor,
\qquad
k=\big\lceil C_{\mathrm{adv}}\frac{\sqrt d}{\log d}\big\rceil.
\]
Now add arbitrary incoming adversarial edges to at most $q$ neurons in each block of $\mathcal H$, with all other neurons unchanged. If, at time $t=0$, each neuron is independently flipped with probability $p\in(0,1/2)$ from the initial state $c$, then with probability $1-o(1)$ as $n\to\infty$, the resulting trajectory returns to the $(k,d)$-weak limit cycle associated with the original monochromatic orbit through $c$, with respect to the canonical block partition.
\end{theorem}

\begin{proof}[Proof Sketch.]
    Theorem \ref{thm:ema2016} asserts that the majority process converges even in the presence of a sufficiently small multiple of $\sqrt{d}/\log d$ adversarial agents, and that the resulting almost-consensus has at most a constant multiple of $\sqrt{d}/\log d$ exceptions. Reusing the same union bound over the $n/d$ blocks as in Theorem \ref{thm:main_reg_sparse}, the condition $\varepsilon(1+\xi)>1$ upgrades this per-block high-probability estimate to an overall $1-o(1)$ bound. See Appendix \ref{sec:main_results} for details.
\end{proof}

\begin{remark}
Appendix \ref{sec:main_results} also contains global convergence statements for arbitrary initial configurations. 
In the dense construction of Theorem \ref{thm:main}, every trajectory reaches a limit cycle deterministically after at most one full turn around the longest block cycle. 
Under the hypotheses of Theorem \ref{thm:main_reg_sparse}, any fixed initial configuration quickly reaches some limit cycle with high probability. 
Under the adversarial perturbations of Theorem \ref{thm:adv}, the same argument yields fast convergence to a $(k,d)$-weak limit cycle; see Theorem \ref{thm:global_cycle}.
\end{remark}

\section{Experiments}
    \label{sec:Expt}
    In this section, we present finite-size simulations illustrating the recovery mechanisms behind our theoretical results. We begin with a single sparse block cycle of length $\ell = 200$, with each neuron receiving $h = 3$ inputs. The initial state is assigned as in Step 1 of our construction (Theorem~\ref{thm:main_reg_sparse}); i.e., each neuron in a given block is assigned the same sign. For simplicity, we take all block signs to be $1$, so this experiment isolates the denoising mechanism rather than the period-length effect of nonconstant block sequences. Each neuron in the block cycle is then flipped independently with probability $p$, after which the Hopfield dynamics evolve. We record whether the initial state is recovered before a full cycle completes. Repeating this procedure $50$ times, we estimate the probability of revival of the initial state for varying block sizes $d$ and flip probabilities $p$. The results are shown in Figure~\ref{fig:expt}(a).

Next, we consider a 3-sparse block cycle architecture consisting of two block cycles of lengths $\ell_1 = 150$ and $\ell_2 = 200$, each with block size $d = 10000$, and add random extra incoming connections. We randomly select $k$ neurons in each block and connect them to randomly chosen neurons in the Hopfield network (with edges directed toward these $k$ neurons). These extra connections are not adversarial: they are sampled randomly, and the selected target neurons are ignored in the weak recovery criterion. Thus increasing $k$ need not monotonically hurt revival; at high noise, random extra inputs can even help by providing additional samples from a population that is already biased toward the correct sign, as seen in the highest-noise curve of Figure~\ref{fig:expt}(b). The initial state is again assigned according to Step 1 of our construction (Theorem~\ref{thm:main_reg_sparse}), as before. Each neuron is then flipped independently with probability $p$, and the Hopfield dynamics are run to determine whether the initial state is recovered before a full round completes, ignoring the neurons that receive additional inputs. Repeating this process $50$ times estimates the convergence probability for different values of $k$ and $p$. The corresponding results are shown in Figure~\ref{fig:expt}(b). Additional implementation details appear in Appendix~\ref{sub_app:Expts}. As an additional stress test motivated by Theorem~\ref{thm:adv}, we consider a block cycle with ``truly adversarial nodes''---nodes that update opposite to the majority rule. We observe revival of the initial state in this setting as well; this experiment is detailed in Appendix~\ref{sub_app:Expts}. We also performed an experiment related to Theorem~\ref{thm:global_cycle} (Appendix~\ref{subapp:global}), where we show that any random configuration in a block cycle converges to a state in which all neurons within each block share the same sign. In that experiment $\ell$ is fixed while $d$ varies, so the mild deterioration with $d$ is consistent with the proof condition that the cycle length must grow at least logarithmically in $d$ to leave enough time for consensus before a full rotation.


\begin{figure}[ht]
    \centering
    \begin{subfigure}{0.49\textwidth}
        \centering
        \includegraphics[width=\linewidth]{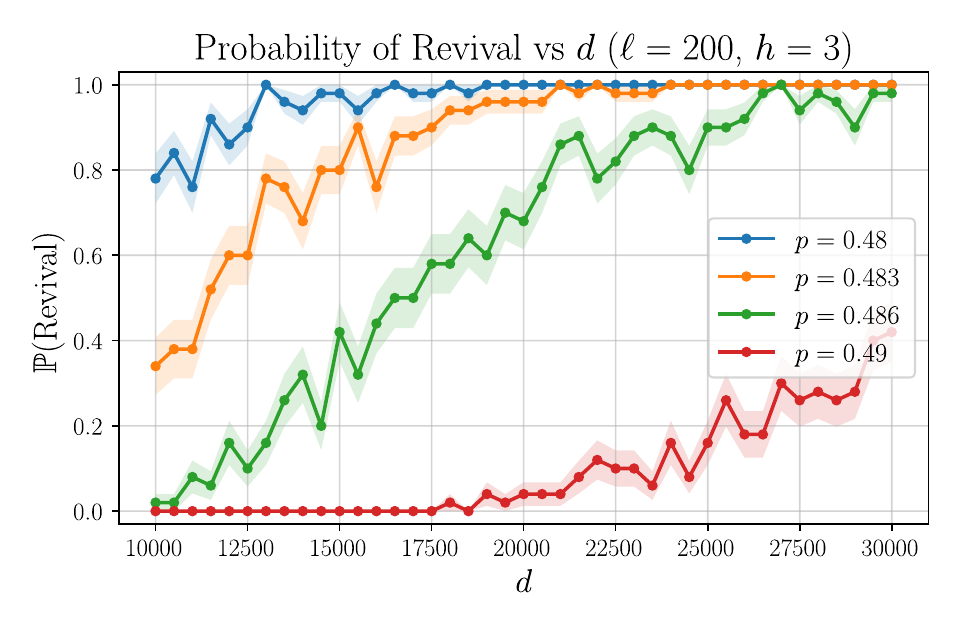}
        \caption{}
    \end{subfigure}
    \hfill
    \begin{subfigure}{0.49\textwidth}
        \centering
        \includegraphics[width=\linewidth]{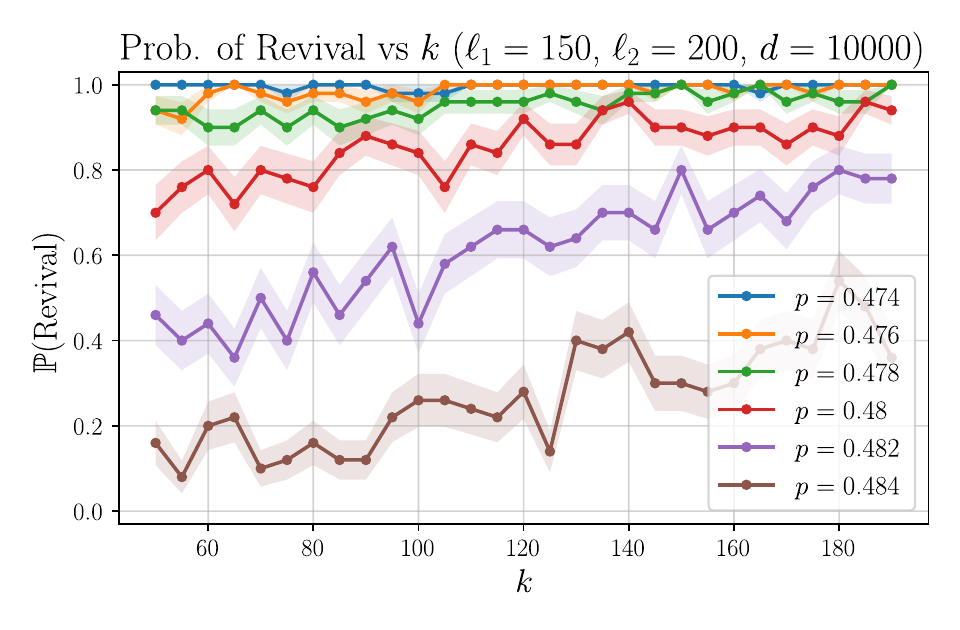}
        \caption{}
    \end{subfigure}

    \caption{Revival probability $\pm$ standard error over $50$ trials. (a) Single sparse block cycle of length $\ell=200$, where each block contains $d$ neurons and each neuron receives $h=3$ random incoming edges from the preceding block. The x-axis varies $d$, and curves correspond to the independent initial flip probability $p$. (b) Sparse block-cyclic architecture with two block cycles of lengths $150$ and $200$, block size $d=10000$, and sparse in-degree $h=3$. The x-axis varies the number $k$ of neurons per block receiving random extra incoming connections, and curves correspond to the initial flip probability $p$.}
    \label{fig:expt}
\end{figure}

\section{Conclusion}
    \label{sec:conc}
    In this work, we showed that classical synchronous asymmetric Hopfield networks can support large, robust sequence memory without nonlinear interactions or continuous states.
A network of $n$ neurons can store $\exp\!\big(\Omega(n/(\log n)^2)\big)$ distinct limit-cycle attractors with periods $\exp\!\big(\Omega(\sqrt n/\log n)\big)$, tolerate independent flips up to $\frac12-\Omega(1/\sqrt{\log n})$ in the dense theorem and up to $\frac12-o(1)$ when $d=\omega(\log n)$, and retain non-trivial basins under sparse and adversarial perturbations.
Thus block-cyclic motifs reveal more power in asymmetric Hopfield dynamics and offer a route to sequence memory via coarse structure rather than fine weight tuning, which is of interest from both a machine learning and biological perspective; learning such motifs is a natural next step.

\paragraph{Limitations.}
Our results are constructive and rely on a prescribed block/block-cycle organization. While the sparse and adversarial-connection results remove all-to-all connectivity and precisely tuned weights, the paper does not provide a local learning rule that forms the topology from unstructured initial connectivity. Our guarantees are also proved for synchronous Little-style updates; whether comparable superpolynomial sequence capacity and noise robustness persist under asynchronous single-neuron Hopfield dynamics remains open. The experiments are synthetic simulations of the theoretical mechanisms rather than evaluations on biological recordings or task-level sequence benchmarks.

\bibliographystyle{plainnat}
\bibliography{references}

\appendix

\section{Aperiodic Binary Sequences}
    \label{sec:necklace}
    \begin{definition}[Möbius function]
The \emph{Möbius function} \(\mu:\mathbb{N}\to\{-1,0,1\}\) is defined by
\[
\mu(n) \;=\;
\begin{cases}
  1 & \text{if $n=1$,}\\[6pt]
  0 & \text{if $n$ is divisible by the square of a prime,}\\[6pt]
  (-1)^r & \text{if $n$ is the product of $r$ distinct primes.}
\end{cases}
\]
\end{definition}

\begin{theorem}
    \label{thm:necklace}
    Let $k\ge 1$ and $\ell\ge 1$ be integers. The number $\mathcal{M}_k(\ell)$ of aperiodic $k$-ary sequences of length $\ell$ is
    \begin{equation*}
        \mathcal{M}_k(\ell) \;=\; \sum_{j\mid \ell} \mu(j)\,k^{\,\ell/j},
    \end{equation*}
    where $\mu$ denotes the Möbius function.
\end{theorem}

\begin{proof}[Proof of Lemma \ref{lem:necklace}]
By Theorem~\ref{thm:necklace},
\[
\mathcal{M}_2(\ell)=\sum_{j\mid \ell}\mu(j)\,2^{\ell/j}
=2^\ell+\sum_{\substack{j\mid \ell\\ j>1}}\mu(j)\,2^{\ell/j}.
\]
For every proper divisor \(j>1\) of \(\ell\), we have \(\ell/j\le \ell/2\), so
\[
\left|\sum_{\substack{j\mid \ell\\ j>1}}\mu(j)\,2^{\ell/j}\right|
\le \sum_{\substack{j\mid \ell\\ j>1}} 2^{\ell/j}
\le \tau(\ell)\,2^{\ell/2},
\]
where \(\tau(\ell)\) is the number of divisors of \(\ell\). Hence
\[
\bigl|\mathcal{M}_2(\ell)-2^\ell\bigr|\le \tau(\ell)\,2^{\ell/2} \implies \mathcal{M}_2(\ell)=2^\ell(1+o(1)),
\]
since \(\tau(\ell)=o(2^{\ell/2})\).
\end{proof}

\section{LCM of Random Subsets}
    \label{sec:lcm}
    In this section, we prove Theorem \ref{thm:cill_ours}. We restate it for the convenience of the reader.
\begin{theorem}
Let $A$ be the set constructed by choosing each element from $\{b,\dots,m\}$ independently with a constant probability $\delta$, for some integer $b=b(m)$ satisfying $b/m = o(1)$.
Then
\begin{equation*}
    \psi(A) \sim m \frac{\delta \log(\delta^{-1})}{1-\delta},
\end{equation*}
asymptotically almost surely as $m \to \infty$.
\end{theorem}

\begin{proof}
Let $R_m\subseteq\{1,2,\dots,m\}$ be the random set obtained by including each
integer independently with the same fixed probability $\delta$.
Write
\begin{equation*}
A_m := R_m\cap\{b,b+1,\dots,m\}
\qquad\text{and}\qquad
D_m := R_m\cap\{1,2,\dots,b-1\}.
\end{equation*}
Then $A_m$ has the distribution appearing in the statement of the theorem, and
$R_m=A_m\cup D_m$ is the corresponding random subset of $\{1,\dots,m\}$.

Using the elementary inequality
\begin{equation*}
\operatorname{lcm}(S\cup T)\le \operatorname{lcm}(S)\operatorname{lcm}(T)
\qquad(S,T\subset\mathbb N\ \text{finite}),
\end{equation*}
we obtain
\begin{equation*}
\psi(R_m)
=\log\operatorname{lcm}(A_m\cup D_m)
\le \psi(A_m)+\psi(D_m).
\end{equation*}
Since $D_m\subseteq\{1,\dots,b-1\}$, we have
\begin{equation*}
\psi(D_m)\le \psi(\{1,\dots,b-1\})
=\log\operatorname{lcm}(1,2,\dots,b-1)
\sim b,
\end{equation*}
and therefore $\psi(D_m)=o(m)$ because $b/m\to 0$.

Consequently,
\begin{equation*}
\psi(A_m)\ge \psi(R_m)-o(m).
\end{equation*}
On the other hand, since $A_m\subseteq R_m$, we trivially have
\begin{equation*}
\psi(A_m)\le \psi(R_m).
\end{equation*}
Hence
\begin{equation*}
\psi(A_m)=\psi(R_m)+o(m).
\end{equation*}

Now Theorem \ref{thm:cill_main} applied to $R_m$ gives
\begin{equation*}
\psi(R_m)\sim m\,\frac{\delta\log(\delta^{-1})}{1-\delta}
\qquad\text{a.a.s.}
\end{equation*}
Combining this with the previous display yields
\begin{equation*}
\psi(A_m)\sim m\,\frac{\delta\log(\delta^{-1})}{1-\delta}
\qquad\text{a.a.s.}
\end{equation*}
as claimed.
\end{proof}

\section{Sampling of Block Cycle Lengths}
\label{sec:sum}
\begin{lemma}
    \label{lem:sum_distinct} 
    Let $M \in \mathbb{N}$ and let $z=z(M)$ satisfy $z\to\infty$. Let $\lrc{Y_i}_{i=1}^z$ be a sequence of i.i.d.\ random variables, each uniformly distributed over $\{1,2,\dots,M\}$, and define
    \[
        S \triangleq \sum_{i=1}^z Y_i.
    \]
    Then, with probability at least $1-\exp\bigl(-\Theta(z)\bigr)$,
    \[
        S = \Theta(zM).
    \]
\end{lemma}

\begin{proof}
    The mean of $S$ equals
    \begin{align}
    \label{eq:E[S]-sum}
        \mu \triangleq \E{S}
        = \E{\sum_{i=1}^z Y_i}
        = \sum_{i=1}^z \E{Y_i}
        = z\frac{M+1}{2}.
    \end{align}
    Since $\E{Y_i}=(M+1)/2$, we have $\mu=\Theta(zM)$.

    Let $\varepsilon>0$. Since each $Y_i \in \{1,2,\dots,M\}$, the summands in $S$ are bounded in $[1,M]$. Hence, by Hoeffding's inequality,
    \begin{align}
        \mathbb{P}\lrp{|S-\mu|\ge \varepsilon \mu}
        &\le 2\exp\left(\frac{-2\mu^2\varepsilon^2}{\sum_{i=1}^z (M-1)^2}\right) \\
        &= 2\exp\left(\frac{-2\mu^2\varepsilon^2}{z(M-1)^2}\right).
    \end{align}
    Using $\mu=\Theta(zM)$, the exponent is $-\Theta(z)$. Therefore,
    \[
        \mathbb{P}\lrp{|S-\mu|\ge \varepsilon \mu}
        = \exp\bigl(-\Theta(z)\bigr).
    \]

    Consequently, with probability at least $1-\exp\bigl(-\Theta(z)\bigr)$,
    \[
        S = (1\pm \varepsilon)\mu = \Theta(zM).
    \]
\end{proof}

\begin{corollary}
    \label{coro:sum}
    Let $\{Y_i\}_{i=1}^z$ be i.i.d.\ uniform on $\{b,\dots,m\}$, and assume $b/m=o(1)$.
    Let
    \[
        M \triangleq m-b+1,
        \qquad
        z \triangleq \lceil 4M\rceil.
    \]
    Then, with probability at least $1-\exp\bigl(-\Theta(M)\bigr)$,
    \[
        \sum_{i=1}^z Y_i = \Theta(m^2).
    \]
\end{corollary}

\begin{proof}
    Let
    \[
        X_i \triangleq Y_i-b+1,\qquad i=1,\dots,z.
    \]
    Then $\lrc{X_i}_{i=1}^z$ are i.i.d.\ uniform on $\{1,2,\dots,M\}$.

    Hence, by Lemma~\ref{lem:sum_distinct}, with probability at least
    $1-\exp\!\lrc{\Theta(-M)}$,
    \[
        \sum_{i=1}^z X_i = \Theta(M^2).
    \]
    Moreover,
    \[
        \sum_{i=1}^z Y_i
        = \sum_{i=1}^z \lrc{X_i+b-1}
        = \sum_{i=1}^z X_i + z(b-1).
    \]
    Since $b/m=o(1)$ and $M=m-b+1$, we have $b=o(M)$, while $z=\Theta(M)$.
    Therefore,
    \[
        z(b-1)=\Theta(M)\cdot o(M)=o(M^2).
    \]
    Consequently,
    \[
        \sum_{i=1}^z Y_i = \Theta(M^2).
    \]
    Finally, because $M=m-b+1=m(1-o(1))$, we have $M=\Theta(m)$ and thus
    \[
        \sum_{i=1}^z Y_i = \Theta(m^2).
    \]
\end{proof}

\begin{lemma}
\label{lem:linear_sampling_domination}
Let $I$ be an interval of integers with $|I|=M$, and let $S^{(\mathrm{rep})}_{4M}$ be the subset of elements of $I$ observed at least once when taking $\lceil 4M\rceil$ samples from $I$ uniformly at random with replacement.
There is a coupling with a subset $S^{(\mathrm{ind})}_\delta\subseteq I$, obtained by including each element of $I$ independently with probability $\delta=1-e^{-2}$, such that
\[
    \mathbb{P}\!\left(S^{(\mathrm{ind})}_\delta\subseteq S^{(\mathrm{rep})}_{4M}\right)
    \ge 1-\exp(-\Theta(M)).
\]
\end{lemma}

\begin{proof}
For each $a\in I$, let $X_a$ be an independent Poisson random variable with mean $2$, and set
\[
    S^{(\mathrm{ind})}_\delta=\{a\in I:X_a\ge 1\}.
\]
Then $S^{(\mathrm{ind})}_\delta$ includes each element independently with probability $\delta=1-e^{-2}$.
Let $T=\sum_{a\in I}X_a$.
Then $T$ is Poisson with mean $2M$, and hence
\[
    \mathbb{P}(T>\lceil 4M\rceil)\le \exp(-\Theta(M)).
\]
Conditional on $T=t$, the vector $(X_a)_{a\in I}$ has the multinomial distribution obtained from $t$ independent uniform samples from $I$.
Therefore, on the event $T\le \lceil 4M\rceil$, we can realize these Poisson counts as the first $T$ samples in a sequence of $\lceil 4M\rceil$ independent uniform samples from $I$, and fill the remaining positions with independent uniform samples.
The set of elements seen in the full sequence has the distribution of $S^{(\mathrm{rep})}_{4M}$ and contains $S^{(\mathrm{ind})}_\delta$ whenever $T\le \lceil 4M\rceil$.
This proves the claim.
\end{proof}

\section{Proof of Main Results}
\label{sec:main_results}
\subsection{Proof of Theorem \ref{thm:main}}
\begin{proof}
We prove statements (1)--(3) for block-cyclic architectures as defined in the statement.

\medskip
\noindent\textbf{Step 1: Constructing limit cycles.}  
Consider the following state of the block-cyclic architecture: arbitrary signs are given to all neurons with one constraint, namely that all neurons in a given block must have the same sign. We claim that such configurations lead to a limit-cycle attractor. To see this, note that all neurons in the same block receive identical inputs at every synchronous update because their incoming signals are identical (and all weights are $+1$).  
Hence, if every block is initially monochromatic (all $d$ neurons share the same sign), then the next update again produces a monochromatic state for each block, which is just a one-step rotation of the previous state.  
The effective state of each block cycle is then a binary string of length $\ell_i$, which is rotated synchronously at each time step.  
Therefore, the block-cyclic architecture exhibits a limit cycle, where each block cycle is an independent rotating binary string. Having shown that our configurations lead to a limit cycle, we now impose another constraint on the configurations: all binary strings on the block cycles are aperiodic, meaning no nontrivial rotation fixes them. In this case, the overall period of the block-cyclic architecture is given by the $\text{lcm}$ of the individual periods, i.e., overall period $= \text{lcm}\{\ell_i\}_{i=1}^z$, where $|A| = z$.
Moreover, any state whose blockwise strict majority signs agree with one of these monochromatic cycle states is mapped in one synchronous update to the next state of the cycle.
Since $d\ge 2$ for all sufficiently large $m$, this gives states outside the cycle in its basin, so the cycle is a limit-cycle attractor.
Let $\mathcal A$ be the family of limit-cycle attractors obtained from these aperiodic assignments.

\medskip
\noindent\textbf{Step 2: Counting limit cycles (proof of (1)).}  
We start by counting how many such distinct limit-cycle attractors are there. For a single block cycle of length $\ell$, the number of aperiodic binary labelings equals the number $\mathcal{M}_2(\ell)$ of binary aperiodic sequences. By Lemma \ref{lem:necklace}, we have
\begin{equation*}
    \mathcal{M}_2(\ell)={2^\ell}(1+o(1)).
\end{equation*}

\noindent Since the block cycles are disjoint, there are $\prod_{i=1}^z \mathcal{M}_2(\ell_i)$ aperiodic monochromatic states of this form.
These are states rather than global orbits: the same limit cycle is counted once for each point on its orbit.
Since the block strings are aperiodic, every such global orbit has period
\(P=\operatorname{lcm}(\ell_1,\dots,\ell_z)\).
Therefore,
\begin{equation*}
    N_{\mathrm{cycles}}:=|\mathcal A| = \frac{1}{P}\prod_{i=1}^z \mathcal{M}_2(\ell_i).
\end{equation*}
Hence we have
\begin{equation*}
    N_{\mathrm{cycles}} = \frac{2^{N_{\mathrm{blocks}}}(1+o(1))^z}{P}, 
    \qquad N_{\mathrm{blocks}}:=\sum_{i=1}^z\ell_i.
\end{equation*}

\noindent Taking $\log$ on both sides and using $N_{\mathrm{blocks}}=n/d$,
\begin{equation*}
    \log N_{\mathrm{cycles}} \ge (\log 2)\frac{n}{d} +z\log(1+o(1)) - \log P.
\end{equation*}
Since each $\ell_i\le m$, we have \(P\le \operatorname{lcm}(1,\dots,m)\), and hence \(\log P\le \psi(m)=O(m)\).
By Lemma \ref{lem:sum_distinct}, we have $m=\Theta((n/d)^{\frac{1}{2}})$. Since $d=\lceil K(\log m)^2\rceil$ and $n=\Theta(m^2d)$, it follows that $\log m=\Theta(\log n)$ and hence $d=\Theta((\log n)^2)$. Also noting that $z=\Theta(m)=\Theta((n/d)^{\frac{1}{2}})$ and $N_{\mathrm{blocks}}=n/d = \Theta(n/(\log n)^2)$, the dominant term is $(\ln 2)n/d$, and the latter two terms are negligible.
Hence there exists $c_1>0$ such that
\begin{equation*}
    N_{\mathrm{cycles}}\ge \exp\!\Big(c_1\frac{n}{(\log n)^2}\Big),
\end{equation*}
proving (1).

\medskip
\noindent\textbf{Step 3: Lower bound on the period (proof of (2)).}  
Let 
\begin{equation*}
    P:=\operatorname{lcm}(\ell_1,\dots,\ell_z)
\end{equation*}
be the least common multiple of the block-cycle lengths.
Recall that the multiset $\{\ell_i\}_{i=1}^z$ is chosen uniformly at random with replacement from $\{1,2,\dots,m\}$ with $z = \lceil 4m\rceil$.
Consider the set $Q$ formed by ignoring repeated elements in the multiset $\{\ell_i\}_{i=1}^z$.
By Lemma~\ref{lem:linear_sampling_domination}, this construction contains, with probability $1-\exp(-\Theta(m))$, a set that includes each integer in $\{1,\dots,m\}$ independently with probability $\delta=1-e^{-2}$.
Since adding elements to a set cannot decrease its $\mathrm{lcm}$, Theorem~\ref{thm:cill_main} implies that, with high probability,
\begin{equation*}
    \log P \ge c' m
\end{equation*}
for some $c'>0$ depending only on $\delta$.  
With $m=\Theta((n/d)^{\frac{1}{2}})$ (Lemma \ref{lem:sum_distinct}), this yields
\begin{equation*}
    P \ge \exp\!\Big( c_2\Big({\frac{n}{d}}\Big)^{1/2}\Big)
    = \exp\!\Big( c_2\frac{\sqrt n}{\log n}\Big),
\end{equation*}
proving (2).

\medskip
\noindent\textbf{Step 4: Noise robustness (proof of (3)).}  
Fix any $\mathcal C\in\mathcal A$ and any state $x\in\mathcal C$.
Note that due to the majority updating rule, a block only gives a different signal to its successor block in a block cycle if more than half of its neurons are flipped. Assume that each neuron is flipped independently with probability $p_n \le \frac12-\sqrt{\frac{\log n}{d}}$, and compute the probability that there exists no block with half or more of its neurons flipped.
Let $p_n \le \frac12-\sqrt{\frac{\log n}{d}}$, and let $Y\sim\mathrm{Binomial}(d,p_n)$ be the number of flipped neurons in a given block.  
For $n$, and equivalently $d$ large enough, a Hoeffding's bound gives
\begin{equation*}
    \mathbb{P}(Y \ge d/2) \le \exp\!\left(-2\frac{(d/2-dp_n)^2}{d}\right) \le e^{-2\log n}=n^{-2},
\end{equation*}
There are $N_{\mathrm{blocks}}=n/d$ blocks, so by a union bound the probability that any block has a majority flipped is at most
\begin{equation*}
    \frac{n}{d}\cdot n^{-2}=\frac{1}{dn}=o(1).
\end{equation*}
Thus, with probability $1-o(1)$ no block changes its majority sign, and the next synchronous update maps the perturbed state exactly to \(F(x)\in\mathcal C\).  
The network therefore returns to $\mathcal C$, establishing (3) and completing the proof.
\end{proof}

\subsection{Co-prime block-cycle lengths}
\label{sub_app:coprime}

In our construction in Theorem \ref{thm:main}, the block-cycle lengths $\{\ell_1,\ell_2,\dots,\ell_z\}$ are chosen uniformly at random from $\{1,2,\dots,m\}$. In this subsection, we show that even if these lengths are chosen more carefully so as to be pairwise co-prime, we do not gain a huge improvement.

\begin{remark}
\label{rem:coprime}
If the block-cycle lengths $\{\ell_1,\ell_2,\dots,\ell_z\}$, with each $\ell_i \in [m]$, are chosen to be pairwise co-prime, then the lower bound on the maximum period improves from
\[
\exp\!\left(\Omega\!\left(\frac{\sqrt n}{\log n}\right)\right)
\quad \text{to} \quad
\exp\!\left(\Omega\!\left(\sqrt{\frac{n}{\log n}}\right)\right).
\]
\end{remark}

To see this, observe that the largest possible size of such a set is at most $\pi(m)+1$, where $\pi(m)$ denotes the number of primes up to $m$. By the Prime Number Theorem, the sum $S(m)$ of the elements in this set satisfies
\begin{equation*}
S(m) \sim \frac{m^2}{\log m}.
\end{equation*}
Since the number of neurons is
\begin{equation*}
n \sim \frac{m^2 d}{\log m},
\end{equation*}
it follows that
\begin{equation*}
m \sim \sqrt{\frac{n}{d}\log\!\left(\frac{n}{d}\right)}.
\end{equation*}
Finally, since the logarithm of the maximum period $P$ scales as $\log P \sim m$ and $d \sim \log^2 n$, we obtain
\begin{equation*}
\log P \sim \sqrt{\frac{n}{d}\log\!\left(\frac{n}{d}\right)}
\sim \sqrt{\frac{n}{\log n}}.
\end{equation*}

\subsection{Proof of Theorem \ref{thm:main_reg_sparse}}

\begin{proof}
    To prove Theorem \ref{thm:main_reg_sparse}, we notice its similarity to Theorem \ref{thm:main}. We will see that the first two parts follow by a very similar argument to Theorem \ref{thm:main}. Then we prove the third part.
    
    \medskip
    \noindent\textbf{Step 1: Constructing limit cycles.}  The aperiodic monochromatic assignments used in the dense proof still define limit cycles for the sparse block-cyclic architecture, because the sparse edges transmit the same block value whenever the preceding block is monochromatic.
    We assume aperiodicity here as well, and let $\mathcal A$ denote the corresponding family of attractors.
    By Corollary~\ref{coro:sum}, with probability $1-o(1)$ the total number of blocks satisfies $N_{\mathrm{blocks}}=\Theta(m^2)$.
    Since $d=\lceil m^{2\varepsilon/(1-\varepsilon)}\rceil$ and $n=dN_{\mathrm{blocks}}$, on this event we have $d=\Theta(n^\varepsilon)$ and $n/d=\Theta(n^{1-\varepsilon})$.
    We work on this event throughout the proof.

    \medskip
\noindent\textbf{Step 2: Counting limit cycles (proof of (1)).} As in Step 2 of the proof of Theorem \ref{thm:main}, the count of aperiodic monochromatic states must be divided by the global orbit period \(P=\operatorname{lcm}(\ell_1,\dots,\ell_z)\).
Since \(P\le \operatorname{lcm}(1,\dots,m)\), we have \(\log P=O(m)=o(n/d)\).
Thus the logarithm of $N_{\mathrm{cycles}}:=|\mathcal A|$ is still dominated by $n/d$.
Since here we have $d=\Theta(n^\varepsilon)$, there exists $c_1>0$ such that
\begin{equation*}
    N_{\mathrm{cycles}}\ge \exp\!\Big(c_1n^{1-\varepsilon}\Big),
\end{equation*}
proving (1).

\medskip
\noindent\textbf{Step 3: Lower bound on the period (proof of (2)).}
Let \(I=\{b,\dots,m\}\) and \(M=|I|=m-b+1\).
The multiset of block-cycle lengths is formed by taking \(\lceil 4M\rceil\) uniform samples from \(I\) with replacement.
By Lemma~\ref{lem:linear_sampling_domination}, after ignoring repeats, the sampled lengths contain an independently sampled subset of \(I\) with constant inclusion probability with high probability.
Since adding elements to a set cannot decrease its \(\mathrm{lcm}\), Theorem~\ref{thm:cill_ours} implies that
\begin{equation*}
    P\geq \exp(c_2 m).
\end{equation*}
On the high-probability event from Step 1, \(m=\Theta((n/d)^{1/2})\).
Since here we have $d=\Theta(n^\varepsilon)$, we get
\begin{equation*}
    P\geq \exp(c_2(n^{1-\varepsilon})^{1/2}),
\end{equation*}
proving (2).

\medskip
    \noindent\textbf{Step 4: Noise robustness (proof of (3)).} For $d$ neurons in any block of any given block cycle consider $d$ ``imaginary agents", undergoing a majority dynamics with 2 opinions. For instance, consider a block with some distribution of states (opinions). At the next unit of time, these opinions will be transmitted to the next block, as the Hopfield update rule is a similar majority update rule, and this will be the new state of the ``imaginary agents". Hence we need whichever opinion is popular among these $d$ agents to reach consensus before completing the cycle. The only difference between Hopfield dynamics and majority dynamics is that in case of a tie, majority dynamics chooses an opinion uniformly at random, whereas Hopfield dynamics preserves the old state. Since $h$ is odd by assumption, they are exactly the same here. Also note that in consensus dynamics, the agents are chosen with replacement, which we handle by setting the weights equal to a node’s multiplicity (see Definition \ref{def:spars_sup_nodes_conn}).
    
    \noindent According to Theorems \ref{thm:hmajo_hierachy} and \ref{thm:large_bias}, the agents should reach consensus (the state before they were flipped) in at most $C \log d$ units of time (this is why we keep the cycles longer than $C\log d$) given that the initially popular opinion has at least $d/2+\lambda_1\sqrt{d\log d}$ supporters.
    Since this happens for each block, the entire block cycle must come back to its original state (before flipping) in at most $C_1\log d$ time.
    We will now calculate the probability of the event that each original block opinion has at least $d/2+\lambda_1\sqrt{d\log d}$ supporters after the random flips.

    Let $X_i$ be the number of neurons that remain unflipped in block $i$. Each neuron stays unflipped with probability $1-p = \frac12 + \eta$, so
    \begin{equation*}
        X_i \sim \mathrm{Bin}\left(d, \frac12 + \eta\right).
    \end{equation*}
    Hence the expectation is
    \begin{equation*}
        \mu := \mathbb{E}[X_i] = d\left(\frac12 + \eta\right) = \frac{d}{2} + \eta d.
    \end{equation*}
    Now take $\eta = C_p\sqrt{\frac{\log n}{d}}$. Then
    \begin{equation*}
        \mu = \frac{d}{2} + C_p\sqrt{d \log n}.
    \end{equation*}
    Since $d=\Theta(n^\varepsilon)$, there is a constant $C_\lambda>0$, depending only on $\lambda_1$ and $\varepsilon$, such that for all sufficiently large $n$,
    \begin{equation*}
    \lambda_1 \sqrt{d \log d}
    \le C_\lambda \sqrt{d \log n}.
    \end{equation*}
    Thus the gap is at least
    \begin{equation*}
        \mu - \left(\frac{d}{2} + \lambda_1 \sqrt{d \log d}\right)
        \ge \left(C_p - C_\lambda\right)\sqrt{d \log n}.
    \end{equation*}
    Choose $C_p>C_\lambda$, as allowed by the statement.
    Let
    \begin{equation*}
        a := \left(C_p - C_\lambda\right)\sqrt{d \log n}.
    \end{equation*}
    By Hoeffding's inequality,
    \begin{equation*}
        \mathbb{P}\left(X_i \le \mu - a\right)
        \le \exp\left(-\frac{2a^2}{d}\right).
    \end{equation*}
    Substituting $a$,
    \begin{equation*}
        \mathbb{P}\left(X_i \le \mu - a\right)
        \le \exp\left(-2\left(C_p - C_\lambda\right)^2 \log n\right)
        = n^{-2\left(C_p - C_\lambda\right)^2}.
    \end{equation*}
    Since $\frac{n}{d}=\Theta(n^{1-\varepsilon})$, there is a constant $C_N>0$ such that the union bound gives
    \begin{equation*}
    \mathbb{P}\left(\exists i : X_i \le \frac{d}{2} + \lambda_1 \sqrt{d \log d}\right)
    \le C_N n^{1-\varepsilon} \cdot n^{-2\left(C_p - C_\lambda\right)^2}.
    \end{equation*}
    For sufficiently large $C_p$, this probability tends to $0$. Hence with high probability, for all blocks $i$,
    \begin{equation*}
    X_i > \frac{d}{2} + \lambda_1 \sqrt{d \log d}.
    \end{equation*}

    \noindent Hence we have ensured the assumptions of Theorems \ref{thm:hmajo_hierachy} and \ref{thm:large_bias}. Suppose an $h$ majority process converges in at most $c\log d$ time with probability say $1-1/d^{\xi}$. A union bound over all blocks gives the failure probability
    \begin{equation*}
        \frac{1}{d^{\xi}}\cdot \frac{n}{d}=O\left(n^{1-\varepsilon(1+\xi)}\right)=o(1),
    \end{equation*}
    since $\varepsilon(1+\xi) > 1$. Now choose $c_3$ such that $c_3\log m \geq c\log d$, which ensures that even for the smallest possible cycle, there is enough time to return to consensus, which proves Theorem \ref{thm:main_reg_sparse}.
\end{proof}
\subsection{Proof of Theorem \ref{thm:adv}}
\begin{proof}
By the same calculation as in the proof of Theorem~\ref{thm:main_reg_sparse}, with probability $1-o(1)$, for each block the initial perturbation preserves a bias of at least $\lambda_1 \sqrt{d \log d}$ towards the original block opinion.
Let \(q=\lfloor \gamma \sqrt{d}/\log d \rfloor\) and \(k=\lceil C_{\mathrm{adv}}\sqrt d/\log d\rceil\), as in the theorem statement.
Since only at most $q$ neurons in each block receive adversarial incoming edges, choosing $\gamma$ sufficiently small ensures that the adversarial budget is within the range allowed by Theorem~\ref{thm:ema2016}.

We may therefore apply Theorem~\ref{thm:ema2016} to each block. Under this bias condition and with at most $q$ adversarially affected neurons, the majority dynamics converges back to the original majority state up to at most \(k\) exceptions within $O((\log d)^2)$ rounds, after choosing \(C_{\mathrm{adv}}\) large enough to dominate the hidden constant in Theorem~\ref{thm:ema2016}. The failure probability is at most $d^{-\xi}$ for some constant $\xi>0$. Reusing the same union bound as above over the $n/d=\Theta(n^{1-\varepsilon})$ blocks, the total failure probability is at most
\begin{equation*}
    \frac{1}{d^{\xi}}\cdot \frac{n}{d}=O\left(n^{1-\varepsilon(1+\xi)}\right)=o(1),
\end{equation*}
by the hypothesis $\varepsilon(1+\xi)>1$. Thus the entire network returns to the $(k,d)$-weak limit cycle associated with the original monochromatic orbit through $c$ with probability $1-o(1)$. The convergence time required is say $c'(\log d)^2$. Choose $c_4$ such that $c_4(\log m)^2 \geq c'(\log d)^2$, which ensures that even for the smallest possible cycle, there is enough time to return to consensus.
\end{proof}

\subsection{Global convergence from arbitrary initial configurations}
\label{subapp:global}
\begin{theorem}
\label{thm:global_cycle}
Fix any initial configuration $x(0)\in\{-1,+1\}^n$.

For every realization of the dense block-cyclic architecture in Theorem \ref{thm:main}, the synchronous Hopfield trajectory starting from $x(0)$ reaches a limit cycle within at most $m$ rounds.

Under the assumptions of Theorem \ref{thm:main_reg_sparse}, with probability $1-o(1)$ over the random construction of the $h$-sparse block-cyclic architecture, the synchronous Hopfield trajectory starting from $x(0)$ reaches a limit cycle within $O(\log d)$ rounds.

Under the assumptions of Theorem \ref{thm:adv}, with at most \(q=\lfloor \gamma \sqrt d/\log d\rfloor\) adversarially perturbed neurons per block, with probability $1-o(1)$ the synchronous Hopfield trajectory starting from $x(0)$ reaches a $(k,d)$-weak limit cycle within $O((\log d)^2)$ rounds, where \(k=\lceil C_{\mathrm{adv}}\sqrt d/\log d\rceil\).
\end{theorem}

\begin{proof}
We begin with the dense block-cyclic architecture. Fix one dense block cycle $S_0,S_1,\dots,S_{\ell-1}$ of length $\ell$. If the sum of the neuron states in $S_j$ at time $t$ is nonzero, then every neuron in $S_{j+1}$ receives the same nonzero input at time $t+1$, so $S_{j+1}$ becomes monochromatic at time $t+1$. If instead the sum of the neuron states in $S_j$ is zero, then every neuron in $S_{j+1}$ receives input zero and therefore keeps its previous state by the convention $\operatorname{sign}(0)=x_i(t)$. Consequently, if every block in the cycle is balanced at time $0$, then the entire block cycle is already a fixed point. Otherwise choose one block whose neuron states have nonzero sum at time $0$. Its successor becomes monochromatic after one round, the next block becomes monochromatic one round later, and so on. After at most $\ell$ rounds every block in that dense block cycle is monochromatic. From that moment onward the block cycle evolves by cyclically shifting these monochromatic signs, and hence it has entered a limit cycle. Since the dense block-cyclic architecture is a disjoint union of such block cycles and every length is at most $m$, the full dense network reaches a limit cycle within at most $m$ rounds.

We first consider the non-adversarial sparse block-cyclic architecture. Fix one block cycle $S_0,S_1,\dots,S_{\ell-1}$ of length $\ell$, and for $0\le t<\ell$ let $Y_t\in\{-1,+1\}^d$ denote the state of $S_t$ at global time $t$. Because adjacent blocks are connected by independent $h$-sparse random bipartite graphs and $h$ is odd, the evolution $Y_0,Y_1,\dots$ agrees with the $h$-majority gossip process on $d$ agents for as long as we do not complete a full turn around the block cycle. By choosing the constant $c_3$ in Theorem \ref{thm:main_reg_sparse} sufficiently large and using $\log d=\Theta(\log m)$, we may assume that every block cycle has length at least $b\ge c\log d$, where $c$ is the convergence constant from Theorem \ref{thm:hmajo_hierachy}. Hence, before any block cycle repeats an edge layer, Theorem \ref{thm:hmajo_hierachy} implies that $Y_t$ becomes monochromatic within $T=O(\log d)$ rounds with failure probability at most $d^{-\xi}$.

Applying the same argument to each of the $n/d=\Theta(n^{1-\varepsilon})$ blocks and taking a union bound, the total failure probability is at most
\begin{equation*}
    \frac{n}{d}\cdot d^{-\xi}=O\left(n^{1-\varepsilon(1+\xi)}\right)=o(1).
\end{equation*}
Therefore, with probability $1-o(1)$, after $T=O(\log d)$ rounds every block is monochromatic. From that time onward each block simply copies the common sign of its predecessor, so the network evolves by cyclically shifting the resulting binary labels on the block cycles. Hence the trajectory has entered a limit cycle.

For the adversarial extension, the same reduction yields on each block cycle an $h$-majority process on $d$ agents with at most \(q=\lfloor \gamma \sqrt d/\log d\rfloor\) adversarially perturbed agents. By choosing $c_4$ sufficiently large in Theorem \ref{thm:adv}, we may assume that every block cycle has length at least $b\ge c'(\log d)^2$, where $c'$ is the convergence constant from Theorem \ref{thm:ema2016}. Hence Theorem \ref{thm:ema2016} implies that within $T'=O((\log d)^2)$ rounds each block reaches consensus on some sign up to at most \(k=\lceil C_{\mathrm{adv}}\sqrt d/\log d\rceil\) exceptional neurons, with failure probability at most $d^{-\xi}$. The same union bound over the $n/d$ blocks again gives total failure probability $o(1)$. Once every block is monochromatic up to those $k$ neurons, the regular part of the network again evolves by cyclically shifting the resulting binary labels. Equivalently, the full trajectory stays within block-max Hamming distance at most $k$ from the corresponding monochromatic orbit for all subsequent times, and therefore it has entered a $(k,d)$-weak limit cycle.
\end{proof}

\section{Background on Opinion Dynamics}
\label{sec:majority}
Opinion dynamics protocols are fundamental models for information spreading and consensus formation in distributed systems. They capture how agents interact locally yet produce global convergence through repeated, randomized communication. We describe the widely studied Gossip model \citep{majority_dynamics_Berenbrink, clementi, Becchetti2015}.

Let $N \in \mathbb{N}$ be the number of agents. The system is represented by a complete graph on $N$ anonymous and identical agents. Time proceeds in discrete, synchronous rounds indexed by $t \in \mathbb{N}_0$. At each time $t$, the global configuration is a vector $C_t \in \{a,b\}^N$ assigning to every agent an opinion from a finite set (here $\{a,b\}$).

In every round, all agents are activated simultaneously. Each agent $u$ independently samples a multiset of $j$ agents $v_1, \dots, v_j$ uniformly at random with replacement from the population. Self-sampling is allowed, i.e., $v_i = u$ is possible. The sampling of different agents and different rounds is mutually independent.

Given the sampled opinions, agent $u$ computes its new opinion according to a deterministic update rule based solely on the multiset $\{C_t(v_1), \dots, C_t(v_j)\}$. In the $j$-majority dynamics, the update rule is defined as follows: agent $u$ adopts the majority opinion among the sampled opinions. In case of a tie (which may occur when $j$ is even), the agent breaks the tie uniformly at random.

All agents perform this sampling and decision step using the configuration $C_t$. After all agents have determined their new opinions, the system performs a synchronous update, producing the next configuration $C_{t+1}$.

Thus, the evolution $(C_t)_{t \ge 0}$ defines a discrete-time stochastic process on $\{a,b\}^N$, where transitions are governed by independent random sampling and simultaneous opinion updates across all agents.

In the study of opinion dynamics, a quantity of interest is the convergence time, i.e., how many rounds are required before all agents reach the same state. We now state a result by \cite{majority_dynamics_Berenbrink} on the convergence time of $j$ majority.

\begin{theorem}[Convergence of $j$-Majority in the Gossip Model]
\label{thm:hmajo_hierachy}
Let $j \ge 3$ be a fixed integer, and let $(C_t)_{t \ge 0}$ be the stochastic process defined above. Then there exists a constant $c = c(j) > 0$ such that, for any initial configuration $C_0 \in \{a,b\}^N$, the process reaches consensus within $O(\log N)$ rounds with high probability. More precisely, if $T$ denotes the (random) convergence time, then
\[
\mathbb{P}\!\left[T \le c \log N \right] \ge 1 - N^{-\xi},
\]
where $\xi = \Omega(1)$.
\end{theorem}

We now state a result by \cite{clementi}, which shows that if the initial configuration has a sufficiently large bias towards one of the opinions, then the initially popular opinion wins at the end with high probability.

\begin{theorem}[Large Bias Implies Correct Consensus]
\label{thm:large_bias}
Let $j \ge 3$ be fixed, and let $(C_t)_{t \ge 0}$ be the $j$-majority gossip process. Let $X_t$ denote the number of agents holding opinion $a$ at time $t$, and assume without loss of generality that $X_0 \ge N/2$.

There exists a constant $\gamma > 0$ such that if the initial bias satisfies
\[
X_0 - (N - X_0) \ge \gamma \sqrt{N \log N},
\]
then the process converges to the consensus configuration in which all agents hold opinion $a$ within $O(\log N)$ rounds with high probability. In particular, if $T$ denotes the convergence time, then
\[
\mathbb{P}\!\left[C_T = (a,\dots,a)\right] \ge 1 - N^{-\Omega(1)}.
\]
\end{theorem}

A \emph{dynamic adversary} with budget $F$ is an adaptive process that, at each round $t$, after observing the configuration $C_t$, may arbitrarily change the opinions of up to $F$ agents (possibly depending on the entire history of the process). The following result by \cite{Becchetti2015} shows robustness under dynamic adversaries.

\begin{theorem}[Robust Convergence under Dynamic Adversaries]
\label{thm:ema2016}
Let $(C_t)_{t \ge 0}$ be the $j$-majority gossip process with $j \ge 3$. Then there exists a constant $\beta > 0$ such that if $F \le \beta \sqrt{N}/\log N$, the following holds: for any initial configuration $C_0 \in \{a,b\}^N$, the process reaches an almost-consensus within $O(\log^2 N)$ rounds with high probability. More precisely, there exists a round $T = O((\log N)^2)$ such that
\[
\left|\{u : C_T(u) = x\}\right| \ge N - O(\sqrt{N}/\log N)
\]
for some opinion $x \in \{a,b\}$, and all but at most $O(\sqrt{N}/\log N)$ agents keep opinion $x$ for any polynomial number of subsequent rounds, despite the adversarial actions.
Equivalently, this event holds with probability at least $1-N^{-\xi}$ for some constant $\xi = \Omega(1)$.

Furthermore, if the initial bias satisfies
\[
\left| \{u : C_0(u)=a\} \right|
-
\left| \{u : C_0(u)=b\} \right|
\ge K\sqrt{N\log N}
\]
for a sufficiently large constant $K>0$, then with high probability the opinion initially held by the majority of agents is the opinion $x$ appearing above.
\end{theorem}

For later use, fix $\lambda_1>0$ large enough that an opinion count of at least \(N/2+\lambda_1\sqrt{N\log N}\) satisfies the bias hypotheses in Theorems~\ref{thm:large_bias} and~\ref{thm:ema2016}.

\section{Additional Related-Work Comparison and Experimental Details}
\label{app:comparison-experiments}
\subsection{Notation Summary}

\begin{table}[H]
\centering
\caption{Summary of the recurring notation used throughout the paper.}
\label{tab:notation-summary}
\small
\renewcommand{\arraystretch}{1.12}
\begin{tabular}{p{0.25\textwidth} p{0.70\textwidth}}
\toprule
\textbf{Notation} & \textbf{Meaning} \\
\midrule
$n$ & total number of neurons in the network. \\
$x(t) \in \{-1,+1\}^n$ & network state at time $t$. \\
$W=(w_{ij})$ & asymmetric weight matrix; $w_{ij}$ is the edge weight from neuron $j$ to neuron $i$. \\
$F$ & synchronous update map induced by the majority-rule dynamics. \\
$\mathcal C$, $P$ & a limit cycle and its period. \\
$\mathcal A$ & family of constructed limit-cycle attractors. \\
$\BCA{d,m}$ & dense block-cyclic architecture with block size $d$ and scale parameter $m$. \\
$A=\{\ell_1,\dots,\ell_z\}$ & sampled multiset of block-cycle lengths. \\
$\ell$, $\ell_i$ & length of a block cycle. \\
$N_{\mathrm{blocks}}$ & total number of blocks in the architecture. \\
$d$ & number of neurons in each block. \\
$m$ & scale parameter controlling the range of sampled block-cycle lengths. \\
$h$ & odd in-degree used in the $h$-sparse construction. \\
$p$ & independent flip probability in the noise model. \\
$k$ & maximum number of adversarially perturbed neurons per block. \\
$d_{\mathrm{BM}}^{(d)}$ & block-max Hamming distance with respect to the canonical block partition. \\
$h$-sparse $\BCA{d,b,m}$ & sparse block-cyclic architecture in which each neuron receives exactly $h$ inputs from the preceding block. \\
$\psi(m)$ & $ \log \mathrm{lcm}\{1 \leq a \leq m\}$\\
$N_{\text{cycles}}$ & total number of limit cycles in a given directed Hopfield network.\\
$d_H$ & hamming distance\\
$\mathcal{B}(\mathcal{C})$ & basin of attraction of a limit cycle $\mathcal{C}$.\\
\bottomrule
\end{tabular}
\end{table}

\subsection{Construction Figures}
\label{app:construction-figures}

\begin{figure}[H]
    \begin{subfigure}[t]{.4\textwidth}
        \centering
        \resizebox{\textwidth}{!}{\begin{tikzpicture}[
    >=Stealth,
    line cap=round,
    line join=round,
    box/.style={
        rounded corners=6pt,
        draw=blue!70!black,
        fill=blue!5,
        line width=1pt
    },
    neuron/.style={
        circle,
        draw=orange!70!black,
        fill=orange!25,
        line width=1.1pt,
        minimum size=12mm,
        inner sep=0pt
    },
    dot/.style={
        circle,
        fill=gray!75,
        draw=gray!75,
        minimum size=2.2mm,
        inner sep=0pt
    },
    conn/.style={
        ->,
        very thick,
        draw=blue!70!black,
        opacity=0.75
    }
]

\draw[box] (0,0) rectangle (1.7,6);
\node[neuron] (L1) at (0.85,5) {};
\node[neuron] (L2) at (0.85,3.5) {};
\node[neuron] (L3) at (0.85,1) {};

\node[dot] at (0.85,2.6) {};
\node[dot] at (0.85,2.2) {};
\node[dot] at (0.85,1.8) {};

\draw[box] (5,0) rectangle (6.7,6);
\node[neuron] (R1) at (5.85,5) {};
\node[neuron] (R2) at (5.85,3.5) {};
\node[neuron] (R3) at (5.85,1) {};

\node[dot] at (5.85,2.6) {};
\node[dot] at (5.85,2.2) {};
\node[dot] at (5.85,1.8) {};

\foreach \a in {L1,L2,L3}{
  \foreach \b in {R1,R2,R3}{
    \draw[conn] (\a) -- (\b);
  }
}

\end{tikzpicture}}
        \caption{Two (fully-)connected blocks}
        \label{fig:supnode_conn}
    \end{subfigure}
    \hfill
    \begin{subfigure}[t]{.5\textwidth}
        \centering
        \resizebox{.7\textwidth}{!}{\begin{tikzpicture}[
    every node/.style={inner sep=0pt, outer sep=0pt},
    cyclearrow/.style={
        ->,
        very thick,
        line cap=round,
        draw=blue!70!black,
        >=Stealth,
        shorten <=2pt,
        shorten >=2pt
    }
]

\definecolor{nodeborder}{RGB}{255,140,0}
\definecolor{dotcolor}{RGB}{80,80,80}

\def\R{3.2cm}
\def\n{8}

\newcommand{\BlockGlyph}{%
\begin{tikzpicture}[baseline=-0.5ex]
  \begin{scope}
    \fill[rounded corners=5pt,
          draw=nodeborder,
          line width=1.2pt,
          top color=orange!25,
          bottom color=yellow!10] (0,0) rectangle (0.6,1.45);
    \foreach \y in {0.28,0.72,1.16}{
      \fill[dotcolor] (0.30,\y) circle (0.14);
      \fill[white, opacity=0.3] (0.26,\y+0.03) circle (0.04);
    }
  \end{scope}
\end{tikzpicture}%
}

\foreach \i in {1,...,\n} {
    \node (N\i) at ({90-(\i-1)*360/\n}:\R) {\BlockGlyph};
}

\begin{scope}[on background layer]
    \fill[blue!5] (0,0) circle (3.8cm);
    \draw[blue!20, line width=1pt] (0,0) circle (3.8cm);
\end{scope}

\foreach \i in {1,...,\n} {
    \pgfmathtruncatemacro{\j}{\i+1}
    \ifnum\i=\n
        \draw[cyclearrow] (N\i) to[bend left=16] (N1);
    \else
        \draw[cyclearrow] (N\i) to[bend left=16] (N\j);
    \fi
}

\end{tikzpicture}}
        \caption{A block cycle with $\ell=8$ blocks and $d=3$ neurons in each block. Each arrow schematically represents two connected blocks as in Figure~\ref{fig:supnode_conn}.}
        \label{fig:sup_cycle}
    \end{subfigure}
    \caption{A block cycle of (fully-)connected blocks.}
\end{figure}
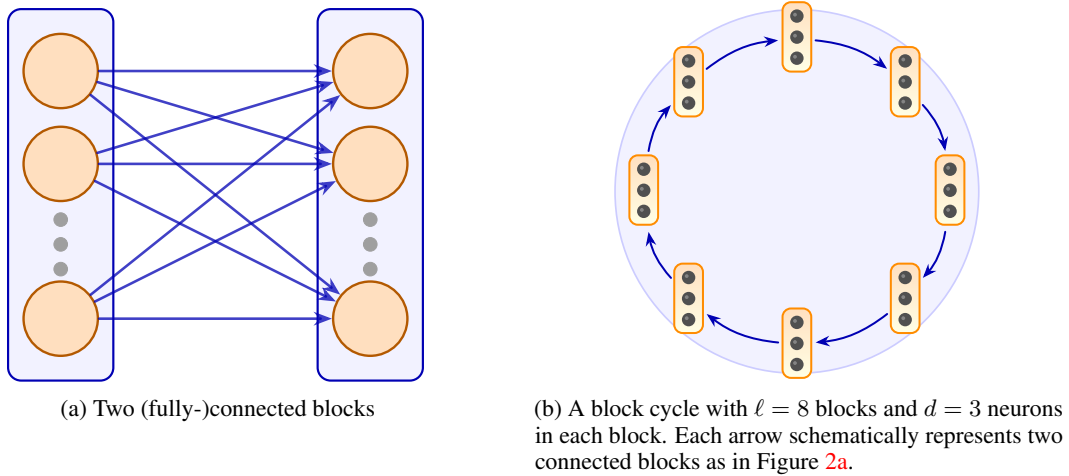

\subsection{Related-Work Comparison}

\begin{table}[H]
\centering
\caption{Comparison with related sequence-memory results in Hopfield-type networks.}
\label{tab:related-work-comparison}
\resizebox{\textwidth}{!}{%
\begin{tabular}{p{2.8cm} p{3.0cm} p{3.0cm} p{3.0cm} p{3.0cm} p{4.0cm}}
\toprule
\textbf{Work}
& \textbf{Model / setting}
& \textbf{Number of cycles}
& \textbf{Cycle length}
& \textbf{Robustness}
& \textbf{Main distinction} \\
\midrule

Bastolla and Parisi~\citep{bastollaAttractorsFullyAsymmetric1997}
& Fully asymmetric random neural networks
& Mean total number grows linearly in \(n\)
& Typical length grows exponentially in \(n\)
& Basin statistics studied, but no designed retrieval robustness
& Long cycles arise statistically in random asymmetric dynamics rather than from a constructive sequence-memory architecture. \\

Hwang et al.~\citep{hwangNumberLimitCycles2019}
& Asymmetric neural networks
& Computes mean number \(n_L\) of cycles of fixed length \(L\)
& Finite \(L\), with an exponential finite-size cutoff for long cycles
& Basins not analyzed in the counting formalism
& Gives a statistical-mechanics characterization of cycle counts and the chaotic long-cycle regime. \\

Hwang et al.~\citep{hwangNumberLimitCycles2020}
& Diluted neural networks
& Exponentially many finite-length cycles when the complexity is positive
& Fixed finite length, mainly \(L=4\) in the fully asymmetric diluted case
& Basins discussed as motivation, not retrieval robustness
& Shows how dilution changes the exponential complexity of short temporal cycles. \\

Zhang et al.~\citep{zhang2013storing}
& Hopfield-type networks with pseudoinverse learning
& Prescribed admissible cycles
& Depends on prescribed cycle
& Retrieval robustness not analyzed
& Gives an admissibility criterion for prescribed cycles and a pseudoinverse construction, with emphasis on network topology. \\

Muscinelli et al.~\citep{muscinelli2017exponentially}
& Standard asymmetric Hopfield network
& Explicitly constructs a maximal orbit
& \(2^n\)
& No nontrivial basin
& Constructs a maximal-length orbit covering the full state space; the orbit cannot be attractive. \\

Chaudhry et al.~\citep{chaudhry2023long}
& Dense associative memory / modern Hopfield-type model
& Not framed as counting multiple limit-cycle attractors
& Sequence capacity is polynomial or exponential, depending on the nonlinearity
& Probabilistic sequence recall below capacity
& Achieves long prescribed sequence recall using nonlinear interactions beyond the synchronous asymmetric Hopfield model. \\

\textbf{This work}
& synchronous asymmetric Hopfield network with binary neurons and majority updates
& \(\exp(\Omega(n/(\log n)^2))\)
& \(\exp(\Omega(\sqrt n/\log n))\)
& Robust to random flips up to \(1/2-o(1)\); persists under sparsification and adversarial connections
 & Simultaneously obtains many long, robust limit-cycle attractors using a simple block / block-cycle architecture. \\

\bottomrule
\end{tabular}
}
\end{table}

\subsection{Additional Experimental Details}
\label{sub_app:Expts}

\paragraph{Details of Experiments in Section~\ref{sec:Expt}.}
In the first set of experiments, we consider a block cycle of length $\ell = 200$, where each block has size $d$ (varied across the experiment), and each neuron receives $h = 3$ inputs, corresponding to the sparsest network topology allowed by our theoretical results. The initial configuration is assigned as in Step~1 of our construction (Theorem~\ref{thm:main_reg_sparse}), so that all neurons within a given block share the same sign. For simplicity we keep all signs +1 while initializing for all further experiments, unless stated; this isolates the block-level denoising mechanism rather than the period-length effect of nonconstant block sequences. We then independently flip each neuron with probability $p$ (also varied across the experiment), after which the Hopfield dynamics are run to test whether the system returns to the initial configuration. This procedure is repeated $50$ times. In each trial, the network topology is resampled: for every neuron, its $3$ neighbors are selected anew. The retrieval probability is then estimated empirically. The results are presented in Figure~\ref{fig:expt}(a). This experiment illustrates reliable recovery of noisy block-constant states while varying network sparsity, noise, and random topology.

Next, we run a random extra-connection experiment motivated by the weak-recovery setting of Theorem~\ref{thm:adv}. Specifically, we consider a $3$-sparse block-cyclic architecture consisting of two block cycles of lengths $\ell_1 = 150$ and $\ell_2 = 200$, each with block size $d = 10000$. We randomly select $k$ neurons in each block and connect them to randomly chosen neurons in the Hopfield network, with edges directed toward these $k$ neurons. These added edges are not adversarial: they are sampled randomly, and the selected target neurons are ignored when evaluating recovery. The value of $k$ is varied across the experiment. The initial configuration is assigned as in the previous experiment, i.e., all neurons in the (+1) state. Each neuron is then independently flipped with probability $p$, after which the Hopfield dynamics are run to determine whether the system recovers the initial state before completing a full cycle, ignoring the neurons that receive additional inputs. This procedure is repeated $50$ times to estimate the revival probability. In each trial, the network topology is resampled, including the random interconnections. The experiments are conducted over a range of values for $k$ and $p$. The results are shown in Figure~\ref{fig:expt}(b). Since the added inputs are random rather than worst-case, they can sometimes help: when the population is already biased toward the correct sign, extra random samples may reinforce that bias, which explains why the highest-noise curve improves for larger $k$.

\paragraph{Further Experiments.}
As a further stress test motivated by Theorem~\ref{thm:adv}, we study the robustness of the dynamics in the presence of extreme adversaries, namely, nodes whose updates are always opposite to the majority rule. We consider a single $3$-sparse block cycle of length $\ell = 200$ with block size $d = 5000$. We randomly select $k$ neurons in the cycle and modify their update rule so that each chosen neuron updates to the opposite of the state prescribed by the standard Hopfield dynamics. The network is initialized with all neurons in the $+1$ state, after which each neuron is independently flipped with probability $p = 0.45$. We then run the Hopfield dynamics and check whether the system returns to the original limit cycle, disregarding the adversarial neurons when evaluating recovery. The probability of successful revival is estimated over $50$ trials, and the results are shown in Figure~\ref{fig:add_expt}(a). These experiments illustrate that the network can retain recovery behavior even in the presence of highly adversarial nodes.

Next, we probe the length condition behind Theorem~\ref{thm:global_cycle} (Appendix~\ref{subapp:global}) by asking whether a randomly initialized block cycle converges to a configuration in which all neurons within each block share the same sign. We consider a single $3$-sparse block cycle with fixed length $\ell = 20$ and block size $d$, where $d$ varies throughout the experiment. Starting from a uniformly random initial configuration, we run the dynamics and measure the probability that the system converges to a state in which every neuron within a given block has identical sign before one full rotation. The probability is estimated over $50$ trials for each value of $d$. The results are presented in Figure~\ref{fig:add_expt}(b). The slight downward trend as $d$ grows is expected because the theorem requires enough time for the $h$-majority process to reach consensus, namely a cycle length growing at least logarithmically with $d$. Thus this fixed-$\ell$ experiment is best read as a finite-time diagnostic: even with a cycle length that does not scale with $d$, success remains high, while the deterioration is consistent with the predicted need for $\ell = \Omega(\log d)$. 



\begin{figure}[H]
    \centering
    \begin{subfigure}{0.46\textwidth}
        \centering
        \includegraphics[width=\linewidth]{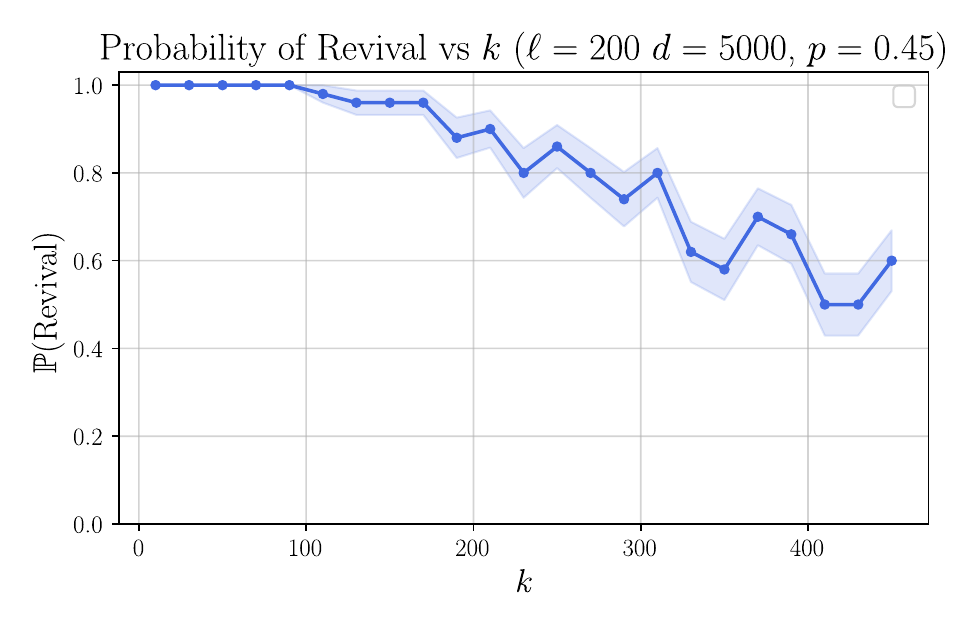}
        \caption{}
        \label{fig:m}
    \end{subfigure}
    \hfill
    \begin{subfigure}{0.46\textwidth}
        \centering
        \includegraphics[width=\linewidth]{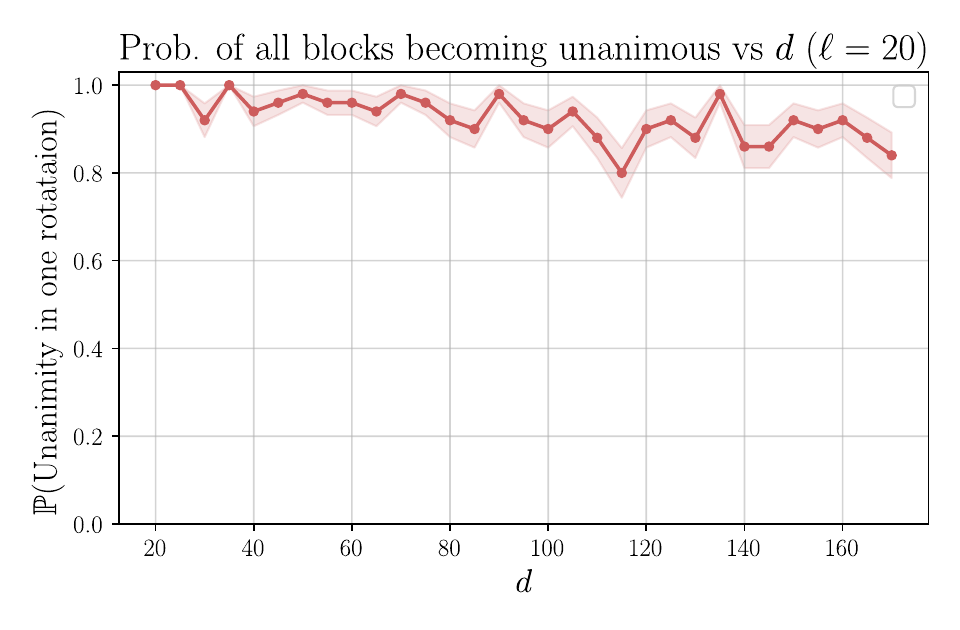}
        \caption{}
        \label{fig:n}
    \end{subfigure}
    \caption{(a) Recovery probability on a sparse block cycle versus number of adversarial neurons, showing robustness. (b) Probability of convergence to a state where each neuron in any given block has the same sign before one full rotation, starting from a random configuration with fixed cycle length $\ell=20$. The mild decline with $d$ reflects that the theorem requires the cycle length to grow at least logarithmically with $d$.}
    \label{fig:add_expt}
\end{figure}

\subsection{Spectral Embedding Plots}
\label{sub_app:spectral_embed}
To provide a geometric view on our architecture, we present spectral embeddings of representative sparse block-cyclic architectures augmented with adversarial connections, capturing the network’s connectivity structure. Nodes are embedded via the leading eigenvectors of the graph Laplacian, which expose the model’s underlying organization. As shown in Figure \ref{fig:spect}, this coarse structure remains largely preserved even after introducing adversarial edges.

\begin{figure}[H]
    \centering
    \begin{subfigure}{0.46\textwidth}
        \centering
        \includegraphics[width=\linewidth]{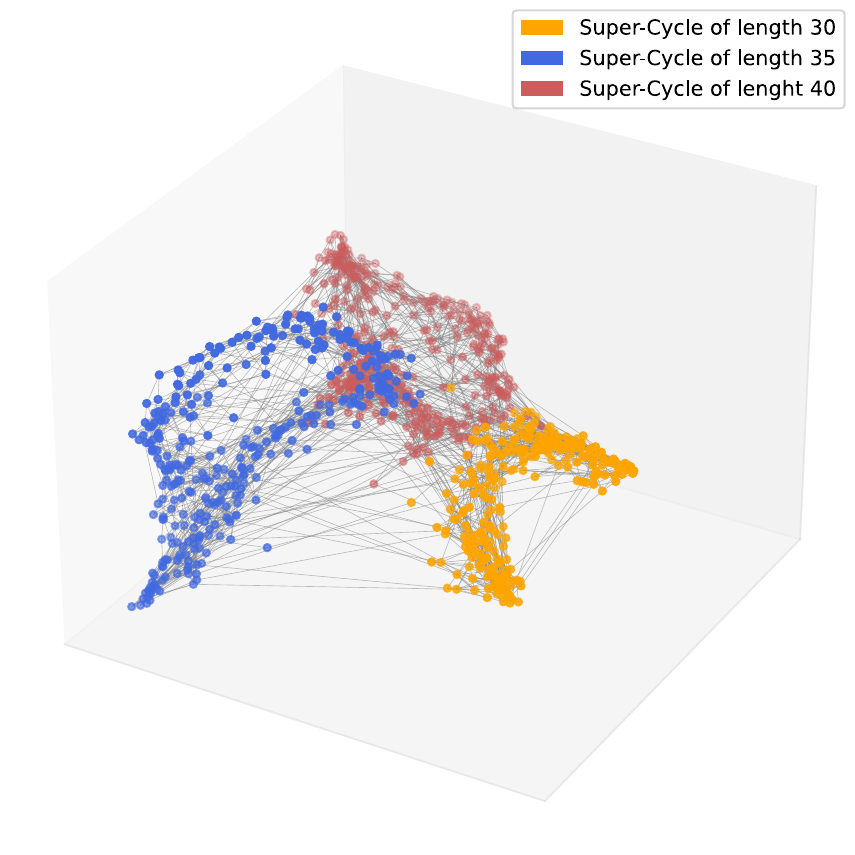}
        \caption{}
        \label{fig:r}
    \end{subfigure}
    \hfill
    \begin{subfigure}{0.46\textwidth}
        \centering
        \includegraphics[width=\linewidth]{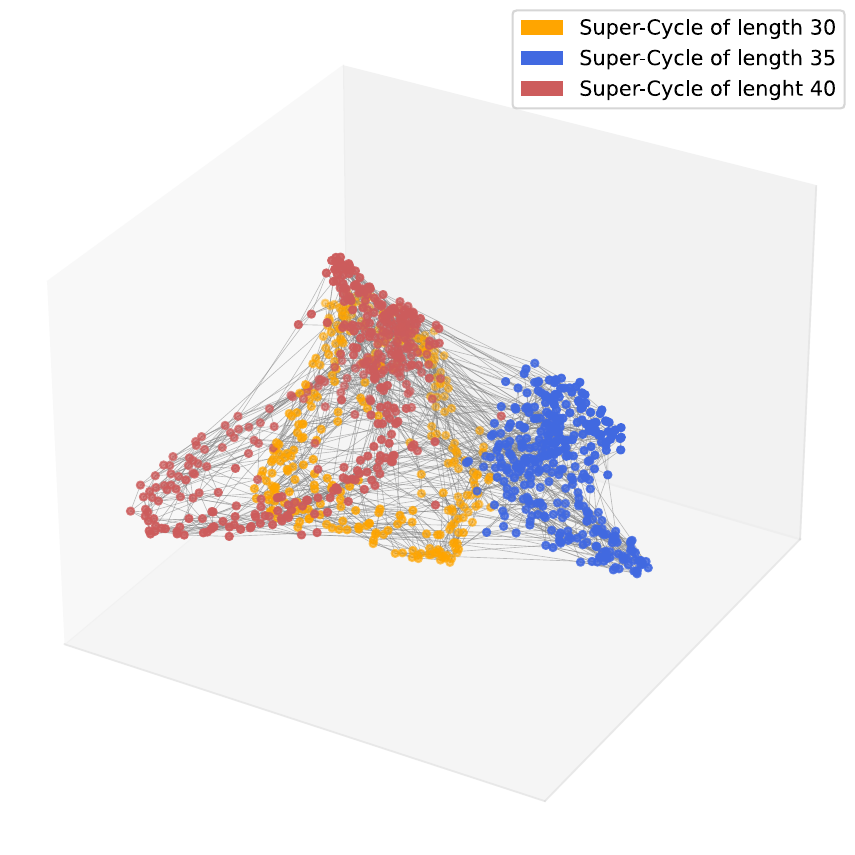}
        \caption{}
        \label{fig:k}
    \end{subfigure}

    \vspace{0.5em}

    \begin{subfigure}{0.46\textwidth}
        \centering
        \includegraphics[width=\linewidth]{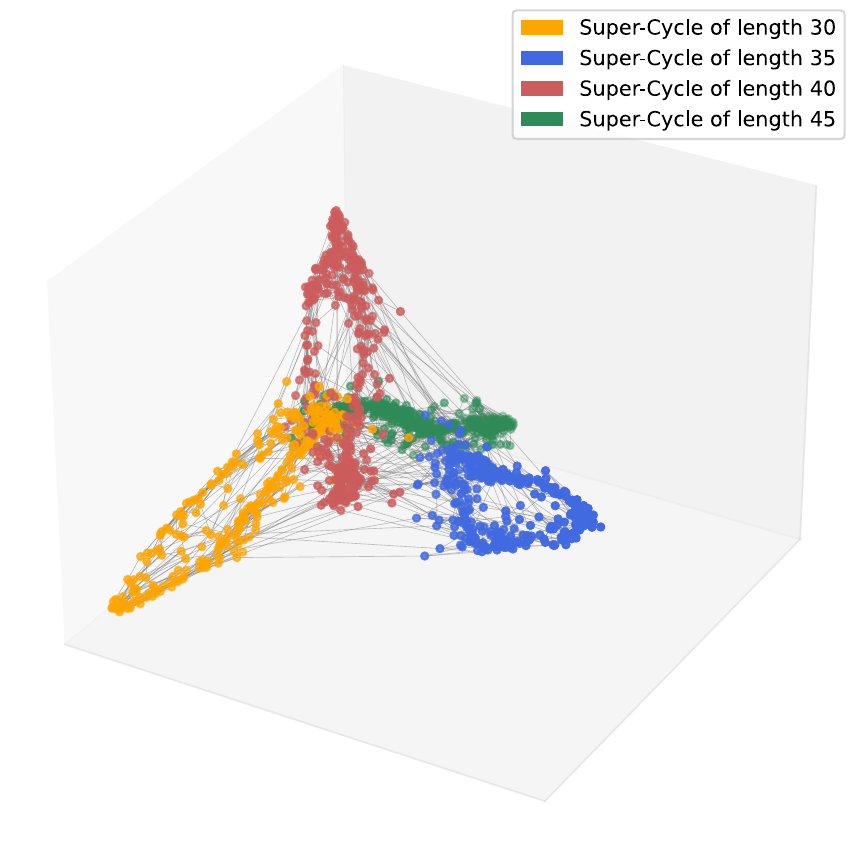}
        \caption{}
        \label{fig:q}
    \end{subfigure}
    \hfill
    \begin{subfigure}{0.46\textwidth}
        \centering
        \includegraphics[width=\linewidth]{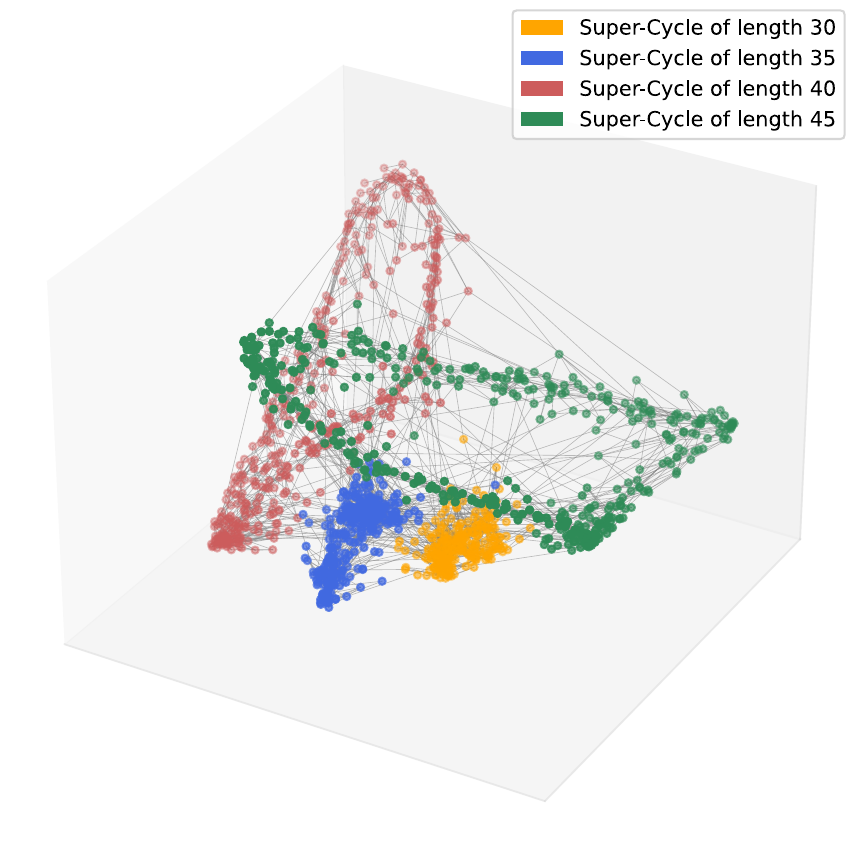}
        \caption{}
        \label{fig:e}
    \end{subfigure}

    \caption{Spectral embeddings of two sparse block-cyclic architectures with adversarial connections: (a) \& (b) block-cycle lengths $30$, $35$, and $40$; (c) \& (d) lengths $30$, $35$, $40$, and $45$. In all panels, each block has size $d=10$ and each neuron receives $h=3$ incoming edges from the preceding block.}
    \label{fig:spect}
\end{figure}

\section{Acknowledgments}
\label{app:ack}
This research has been supported by the French government National Research Agency (ANR) through the UCA JEDI (ANR-15-IDEX-01), EUR DS4H (ANR-17-EURE-004), and the 3IA Côte d’Azur Investments in the Future project with the reference number ANR-23-IACL-0001.

\end{document}